%% file: main.tex
\documentclass{fairmeta}

\usepackage{amsmath}
\usepackage{amssymb}
\usepackage{nicefrac}
\usepackage{amsthm}
\input{math_commands.tex}

\usepackage{array}
\usepackage{enumitem}
\usepackage{float}

\titleformat{\paragraph}[runin]{\normalfont\normalsize\bfseries}{}{0pt}{}
\titlespacing*{\paragraph}{0pt}{1.5ex plus 0.5ex minus .2ex}{0.5em}
\DeclareTextFontCommand{\textbf}{\bfseries}

\definecolor{findingnavy}{RGB}{0,0,128}
\definecolor{navy}{RGB}{0,0,128}
\newtcolorbox{findingbox}[1]{
  colback=gray!10,
  colframe=gray!10,
  boxrule=0pt,
  arc=3pt,
  left=8pt, right=8pt, top=3pt, bottom=3pt,
  before skip=6pt,
  after skip=6pt,
  fonttitle=\bfseries,
  title={#1},
  colbacktitle=gray!10,
  coltitle=black,
  detach title,
  before upper={\tcbtitle}
}

\newtcolorbox[auto counter,number within=section]{algorithmbox}[2][]{
  colback=white,
  colframe=black!55,
  boxrule=0.6pt,
  arc=1pt,
  left=6pt, right=6pt, top=4pt, bottom=4pt,
  before skip=6pt,
  after skip=6pt,
  fonttitle=\bfseries,
  before upper={\raggedright},
  title={Algorithm~\thetcbcounter: #2},
  #1,
}

\newtheorem{proposition}{Proposition}[section]

\title{Scaling Laws for Looped Mixture of Experts}

\author[*]{Yanbei Chen}
\author[\dagger]{Anirudh Goyal}
\author[\dagger]{Raghuraman Krishnamoorthi}
\affiliation[]{Meta AI}
\contribution[*]{Corresponding author}
\contribution[\dagger]{Senior authors}
\date{Sep 30, 2026}
\correspondence{\email{yanbeichen@meta.com}}

\abstract{Looped transformers and Mixture-of-Experts (MoE) offer complementary routes to efficient scaling: recurrence increases computational depth at fixed parameters, while MoE sparsity expands total capacity at fixed active compute. Yet existing scaling laws model recurrence or sparsity in isolation. In this work, we introduce \textit{Loop Scaling Laws}, the first scaling law to jointly model recurrence and sparsity alongside model size and data. At its core is a bounded, sparsity-conditional recurrence mapping that characterizes the effective-parameter gain from looping and how sparsity raises this gain. The laws predict the held-out loss of looped models more accurately than prior alternatives, and recover the standard dense and MoE scaling laws as special cases. Beyond prediction, the fitted laws provide a principled foundation for designing looped MoE models under compute and memory constraints. Downstream evaluations further demonstrate the complementary benefits of the two axes: sparsity delivers ${\sim}3\times$ active-parameter efficiency, recurrence yields ${\sim}2\times$ total-parameter efficiency on reasoning, and joint scaling further advances the performance frontier. As a practical extension, we show these gains hold at trillion-token scale: at matched training compute, a looped MoE with law-derived recurrence matches a ${\sim}2\times$ larger non-looped MoE on the reasoning benchmarks, while enabling test-time scaling through recurrence.}

\begin{document}

\maketitle

\section{Introduction}
\label{sec:intro}

Scaling language models has conventionally relied on larger models and more data~\citep{kaplan2020scaling,hoffmann2022chinchilla}. Beyond model and data, two additional axes enable efficient scaling. Mixture-of-Experts (MoE) introduces sparsity, expanding total capacity at fixed active compute, and is now widely adopted by frontier models from cloud~\citep{deepseek,qwen3} to edge~\citep{chen2026mobilemoe,apple2026afm3}. Recently, looped transformers~\citep{dehghani2018universal} introduce recurrence, reusing shared weights across passes to increase computational depth at fixed stored parameters and delivering substantial gains in reasoning~\citep{saunshi2025reasoning,geiping2025scaling,zhu2025scaling,jeddi2026loopformer}. Looped MoE models naturally combine these complementary axes.

Yet the joint scaling behavior of recurrence and sparsity remains underexplored. Scaling laws for looped transformers capture recurrence, modeling its capacity gain with linear or power-law forms that grow without bound~\citep{prairie2026parcae,schwethelm2026much}. MoE laws capture sparsity, varied through expert count, but omit recurrence~\citep{clark2022unified,ludziejewski2025joint,abnar2025parameters}. Recent and concurrent looped MoE studies examine sparsity while fixing recurrence at two passes in their primary scaling analyses~\citep{lee2026sparse,wang2026smelt}, leaving the interaction between the two axes unmodeled. Consequently, a unified scaling law that jointly characterizes recurrence and sparsity in looped MoE models has yet to be established.

\begin{figure}[!t]
    \centering
    \includegraphics[width=\textwidth]{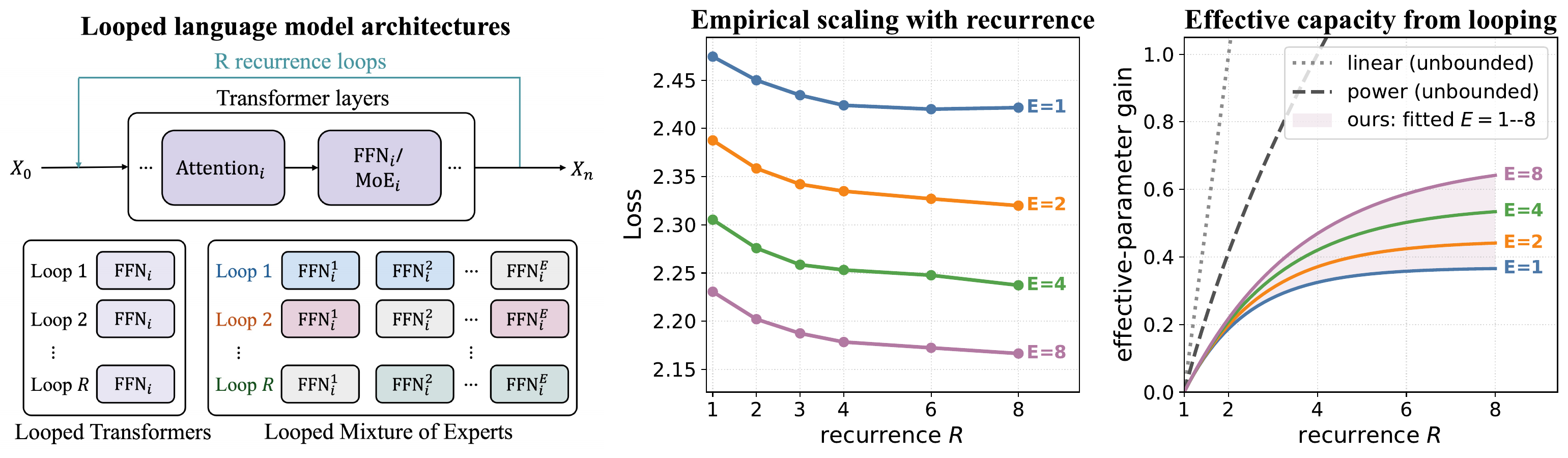}\\
    \makebox[0.43\linewidth]{\small (a)}%
    \makebox[0.31\linewidth]{\small (b)}%
    \makebox[0.26\linewidth]{\small (c)}
\caption{
(a) Dense looped transformers reuse the same FFNs, whereas looped MoE models can route tokens to different experts across recurrent passes.
(b) At fixed active parameters and training tokens, increasing recurrence yields diminishing returns, with the dense model ($E{=}1$) saturating earlier than MoE models ($E{>}1$).
(c) Our Loop Scaling Laws model a bounded effective-parameter gain conditioned on recurrence $R$ and sparsity (varied by expert count $E$), unlike prior linear and power-law mappings that assume unbounded gains.}
    \label{fig:loopmoe}
\end{figure}

Establishing such a law requires understanding how recurrence and sparsity shape the effective capacity gain from looping. This raises two central questions: \textit{\textbf{(Q1)} How much effective capacity does each recurrent pass add, and does this gain grow indefinitely?} \textit{\textbf{(Q2)} Does greater MoE sparsity increase and sustain this gain across recurrent passes?} Two observations in Figure~\ref{fig:loopmoe}(b) point toward the answers. First, increasing recurrence shows diminishing returns: loss drops sharply over the first few passes, and then approaches a plateau within the observed range. Second, greater sparsity, varied by expert count $E$, improves and sustains this gain: with more experts, looping achieves lower loss over more passes. Figure~\ref{fig:loopmoe}(a) explains this interaction intuitively: in a dense looped transformer, every recurrent pass reuses the same weights; whereas in a looped MoE, the router can send a token to different experts across passes, allowing each pass to reach a new set of parameters.

To answer the above questions, we introduce \textit{Loop Scaling Laws}, the first unified law over model size, data, recurrence, and sparsity. The law treats each recurrent pass as adding effective parameters with diminishing returns: each pass adds less than the last, with its total capacity gain approaching a finite asymptote rather than growing without bound. This asymptote also depends on sparsity: more experts raise and sustain the gain over more recurrent passes, so recurrence and sparsity interact within a single form, as illustrated in Figure~\ref{fig:loopmoe}(c). The law subsumes the standard dense, dense-looped, and non-looped MoE scaling laws as special cases. Empirically, it predicts held-out loss more accurately than prior alternatives with unbounded mappings, extrapolating well to unseen recurrence and expert counts.

Beyond prediction, the fitted law provides a principled recipe for looped MoE model design: it selects the compute- and memory-optimal recurrence and expert count under given budgets for resource-constrained deployment. Our empirical results across 14 downstream benchmarks further confirm the complementary gains of sparsity and recurrence. By scaling sparsity, an MoE model can surpass larger dense models with $\nicefrac{1}{3}$ the active parameters, while scaling recurrence enables a looped model to match non-looped models with $\nicefrac{1}{2}$ the total parameters. Jointly scaling both further advances the frontier beyond either axis alone. Together, these results establish joint scaling of recurrence and sparsity as a new parameter-efficient scaling paradigm. We extend this paradigm to practical, trillion-token training: at matched compute, a looped MoE model (0.3B active/1.3B total) with law-derived recurrence matches the reasoning performance of a larger non-looped MoE (0.6B active/2.9B total), trading additional inference compute for approximately $2\times$ parameter efficiency. We also show that the looped MoE can unlock on-demand test-time scaling by varying recurrence at inference. 

\section{Related Work}
\label{sec:related_work}

\paragraph{Looped transformers} (also known as recurrent-depth or recursive transformers) increase computational depth by repeatedly applying shared layers without proportional growth in stored parameters.
Looped architectures have evolved from single-layer recurrence in Universal Transformers~\citep{dehghani2018universal} to full-model recurrence~\citep{zhu2025scaling,giannou2023looped} and middle-block reuse for recursive and latent-reasoning models~\citep{saunshi2025reasoning,zeitoun2026hyperloop,mcleish2025teaching,geiping2025scaling}.
Training objectives range from final-pass next-token prediction loss~\citep{saunshi2025reasoning,geiping2025scaling,mcleish2025teaching} to intermediate-pass supervision~\citep{bae2024relaxed}, self-distillation~\citep{goyal2026elt}, and shortcut consistency~\citep{jeddi2026loopformer}.
Prior work shows that recurrent depth provides parameter-efficient training- and test-time scaling, especially on reasoning tasks~\citep{saunshi2025reasoning,geiping2025scaling,zhu2025scaling,jeddi2026loopformer}. Motivated by these findings, we model recurrence as an explicit scaling axis to characterize how performance scales with looping.

\paragraph{Mixture of experts} (MoE) models increase total capacity without proportionally increasing per-token computation by routing each token to a sparse subset of expert networks~\citep{shazeer2017outrageously,fedus2022switch}.
MoE has become a well-established scaling paradigm across deployment regimes, from open-weight and proprietary frontier models deployed in the cloud, e.g., Mixtral~\citep{mixtral}, DeepSeek-V3~\citep{dai2024deepseekmoe,deepseek}, Qwen3MoE~\citep{qwen3}, and Gemini~\citep{geminiteam2024gemini15}, to resource-constrained models deployed on edge devices, e.g., MobileMoE~\citep{chen2026mobilemoe} and AFM 3 Core Advanced~\citep{apple2026afm3}.
Recent works have also explored looped MoE models, where looped MoE is shown to outperform non-looped dense transformers~\citep{csordas2024moeut}, and MoE effectively improves looping performance~\citep{lee2026sparse} or vice versa~\citep{wang2026smelt}, suggesting the complementary architectural advantages of unifying MoE and looping. We further formulate a unified scaling law that quantifies their scaling benefits in a single functional form.

\begin{table}[!t]
\centering
\caption{Comparison of scaling laws for training LLMs. $\checkmark$ indicates variables explicitly modeled by each law. $N/F$ denotes model size $N$ / training FLOPs $F$. $S/E$ denotes MoE sparsity $S$ (varied by expert expansion $E$).
}
\label{tab:related-scaling-laws}
\scriptsize
\setlength{\tabcolsep}{3.5pt}
\resizebox{0.95\linewidth}{!}{%
\begin{tabular}{@{}llcccclcc@{}}
\toprule
& & \multicolumn{4}{c}{\textbf{Scaling axes}} & &
\multicolumn{2}{c}{\textbf{Optimality}} \\
\cmidrule(lr){3-6}\cmidrule(l){8-9}
\textbf{LLM Category} & \textbf{Scaling law} & $N/F$ & $D$ & $R$ & $S/E$ &
\textbf{Recurrence mapping} & Compute & Memory \\
\midrule
Dense & Chinchilla~\citeyearpar{hoffmann2022chinchilla}
& \checkmark & \checkmark & $\times$ & $\times$
& --
& \checkmark & $\times$ \\
\midrule
\multirow{2}{*}{MoE}
& Unified routed~\citeyearpar{clark2022unified}
& \checkmark & $\times$ & $\times$ & \checkmark
& --
& $\times$ & $\times$ \\
& Joint MoE~\citeyearpar{ludziejewski2025joint}
& \checkmark & \checkmark & $\times$ & \checkmark
& --
& \checkmark & \checkmark \\
\midrule
\multirow{2}{*}{Looped}
& Parcae~(\citeyear{prairie2026parcae}\textcolor{metablue}{/04})
& \checkmark & \checkmark & \checkmark & $\times$
& linear; unbounded$^{\dagger}$
& \checkmark & $\times$ \\
& Iso-Depth~(\citeyear{schwethelm2026much}\textcolor{metablue}{/05})
& \checkmark & \checkmark & \checkmark & $\times$
& power law; unbounded
& \checkmark & $\times$ \\
\midrule
\multirow{3}{*}{Looped MoE}
& Sparse Layers~(\citeyear{lee2026sparse}\textcolor{metablue}{/05})
& \checkmark & $\times$ & $\times$ & $\times$
& fixed recurrence ($R{=}2$)
& \checkmark & $\times$ \\
& SMELT~(\citeyear{wang2026smelt}\textcolor{metablue}{/09})
& \checkmark & \checkmark & $\times$ & \checkmark
& fixed recurrence ($R{=}2$)
& \checkmark & $\times$ \\
& \textbf{Loop scaling laws (ours)}
& \checkmark & \checkmark & \checkmark & \checkmark
& \textbf{bounded; sparsity-dependent}
& \checkmark & \checkmark \\
\bottomrule
\end{tabular}}
\begin{minipage}{\linewidth}
\fontsize{9}{9}\selectfont
$^{\dagger}$Parcae uses the fully unrolled parameter count in its trained-recurrence scaling analysis and separately fits a saturating law over \emph{test-time} depth.
\end{minipage}
\end{table}

\paragraph{Scaling laws} characterize predictable power-law relationships among loss, model size, data, and compute, providing a principled foundation for compute-optimal language model training~\citep{kaplan2020scaling,hoffmann2022chinchilla}.
MoE scaling laws further incorporate routed expert count~\citep{clark2022unified}, expert granularity~\citep{krajewski2024scaling}, sparsity~\citep{abnar2025parameters}, enabling optimization under memory constraints~\citep{ludziejewski2025joint}, efficiency leverage~\citep{tian2026towards}, and on-device constraints~\citep{chen2026mobilemoe}.
More recently, Parcae and Iso-Depth model recurrence for dense looped transformers~\citep{prairie2026parcae,schwethelm2026much}, while concurrent work SMELT analyzes looped MoE models with recurrence fixed in their primary scaling analyses~\citep{wang2026smelt}.
As summarized in Table~\ref{tab:related-scaling-laws}, prior laws model either recurrence or sparsity, but not their joint effect. Our unified MoE loop scaling law jointly models model size, data, recurrence, and sparsity, while recovering the prior laws for dense and MoE models as equivalent reduced forms.

\section{Loop Scaling Laws}
\label{sec:methodology}

\subsection{Preliminaries}
\label{sec:prelim}

\paragraph{Scaling law.} The standard Chinchilla-style scaling law~\citep{hoffmann2022chinchilla,kaplan2020scaling} models the relationship between model parameters $N$, training tokens $D$, and model loss $\mathcal{L}$ as
\begin{equation}
\mathcal{L}(N, D) = A\,N^{\alpha} + B\,D^{\beta} + c ,
\label{eq:chinchilla}
\end{equation}
where $A,B,\alpha,\beta,c$ are fitted coefficients, $\alpha,\beta<0$ are the model and data scaling exponents, $c$ is the irreducible loss. Training compute is $F_{\text{train}}=6ND$, and inference compute is $F_{\text{inf}}=2N$ per token.

\paragraph{Parameter counts and compute.}
A looped model reuses a block of parameters $N_{\text{loop}}$ over $R$ recurrent passes with fixed model parameters $N$. In a forward pass, its unrolled parameter count is
\begin{equation}
N_{\text{unroll}}(R)=N_{\text{act}}+(R-1)\,N_{\text{loop}}.
\label{eq:unroll}
\end{equation}
For dense looped models, the parameters are fixed as $N_{\text{act}}{=}N_{\text{total}}{=}N$. For looped MoE models, the active and total parameters $N_{\text{act}}$ and $N_{\text{total}}$ are fixed. In both cases, $R{=}1$ recovers the corresponding non-looped baseline.
Under full backpropagation through all recurrent passes, the training and per-token inference compute are $F_{\text{train}}=6\,N_{\text{unroll}}(R)\,D$ and $F_{\text{inf}}=2\,N_{\text{unroll}}(R)$, respectively.
While looping recurrence trades compute for quality at fixed parameters, MoE sparsity trades parameters for quality at fixed compute. We explore these complementary scaling axes in the following.

\subsection{Loop Scaling Law}
\label{sec:loop_scaling_law}

\begin{figure}[!t]
\centering
\begin{minipage}[c]{0.60\textwidth}
    \centering
    \includegraphics[width=\linewidth]{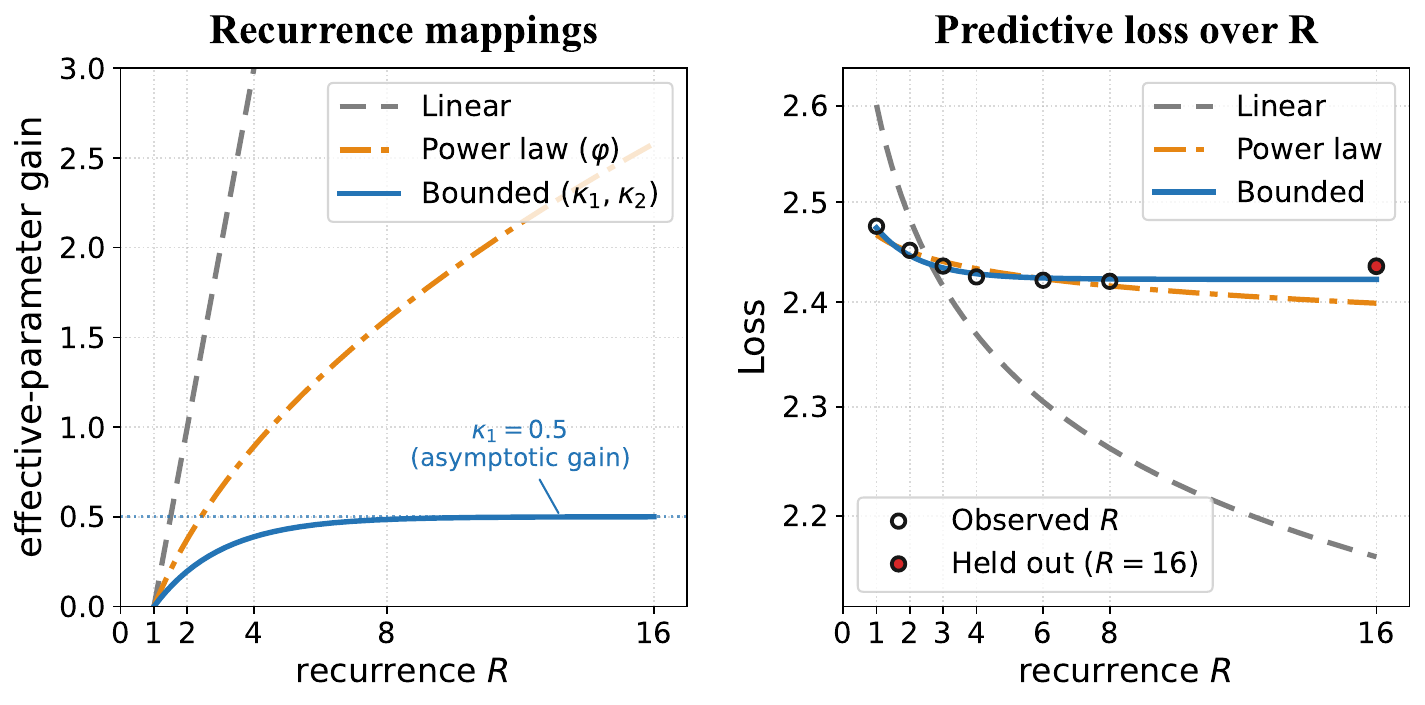}
\end{minipage}\hfill
\begin{minipage}[c]{0.38\textwidth}
    \centering
    {\scriptsize\bfseries Recurrence mappings for $N_{\text{eff}}$\par}
    \scriptsize
    \setlength{\tabcolsep}{1.25pt}
    \renewcommand{\arraystretch}{1.0}
    \resizebox{0.93\linewidth}{!}{%
    \begin{tabular}{@{}l|c|c@{}}
    \toprule
    \textbf{Mapping} & $N_{\text{eff}}-N$ & $R{\to}\infty$ \\
    \midrule
    \textbf{Linear} & $(R{-}1)N_{\text{loop}}$ & $\infty$ \\
    \textbf{Power law} & $(R^{\varphi}{-}1)N_{\text{loop}}$ & $\infty$ \\
    \textbf{Bounded} & $\kappa_1N_{\text{loop}}\!\left(1{-}e^{{-}(R{-}1)/\kappa_2}\right)$ & $\kappa_1N_{\text{loop}}$ \\
    \bottomrule
    \end{tabular}}\\[5pt]
    {\scriptsize\bfseries Held-out prediction RMSE ($\downarrow$)\par}
    \scriptsize
    \setlength{\tabcolsep}{3.5pt}
    \renewcommand{\arraystretch}{1.0}
    \resizebox{0.82\linewidth}{!}{%
    \begin{tabular}{@{}lccc@{}}
    \toprule
    \textbf{Evaluation} & \textbf{Linear} & \textbf{Power law} & \textbf{Bounded} \\
    \midrule
    \textbf{Held-out $R$} & 0.2566 & 0.0313 & \textbf{0.0092} \\
    \textbf{Held-out $N$} & 0.1106 & 0.0211 & \textbf{0.0128} \\
    \textbf{Held-out $D$} & 0.0882 & 0.0073 & \textbf{0.0049} \\
    \bottomrule
    \end{tabular}}
\end{minipage}
\par
\noindent
\makebox[0.30\textwidth]{\small (a)}%
\makebox[0.30\textwidth]{\small (b)}%
\hfill
\makebox[0.38\textwidth]{\small (c)}
\caption{(a) Illustration of effective-parameter gain under linear, power-law, bounded recurrence mappings. (b) Predictive loss by different recurrence mappings: each mapping is fitted on $R\leq8$ and applied to predict the held-out-$R$ up to $R=16$. (c) Formula on recurrence mappings and their effective-parameter gains (top) and their evaluation on held-out recurrence $R$, model $N$, and data $D$ with RMSE scores (bottom).
}
\label{fig:dense-loop}
\end{figure}

The standard scaling law depends only on model $N$ and data $D$, and thus cannot capture the scaling behaviour of looping recurrence $R$. Prior scaling laws for looped transformers address this limitation by replacing $N$ with a recurrence-dependent effective parameter count $N_{\text{eff}}(R)$. Parcae's training scaling law~\citep{prairie2026parcae} uses the fully unrolled parameter count: a linear mapping on $R$; whereas Iso-Depth~\citep{schwethelm2026much} learns a power-law mapping on $R$ with exponent $\varphi$:
\begin{equation}
N_{\text{eff}}^{\text{linear}}(R)=N+(R-1)N_{\text{loop}},
\qquad
N_{\text{eff}}^{\text{power}}(R)=N+(R^\varphi-1)N_{\text{loop}}.
\label{eq:recurrence-mappings-revised}
\end{equation}
$N_{\text{eff}}^{\text{linear}}(R)$ assigns a constant effective-parameter gain $N_{\text{loop}}$ per recurrence; $N_{\text{eff}}^{\text{power}}(R)$ yields diminishing gains with $0{<}\varphi{<}1$. Nevertheless, both mappings are unbounded: $\lim_{R\to\infty}N_{\text{eff}}(R)=\infty$, where increasing recurrence substitutes for parameters indefinitely (Figure~\ref{fig:dense-loop}(a)).

\paragraph{A bounded recurrence mapping.} However, increasing recurrence yields diminishing returns. Empirically, the model loss drops substantially over the earlier passes, and flattens gradually within the evaluated range (Figure~\ref{fig:dense-loop}(b)). This behavior is consistent with weight sharing: each additional pass increases computational depth without introducing new learned parameters, leading to diminishing marginal gains. We therefore introduce the following bounded monotone mapping:
\begin{equation}
N_{\text{eff}}^{\text{bounded}}(R)
=
N+\kappa_1N_{\text{loop}}
\left(1-e^{-(R-1)/\kappa_2}\right),
\qquad \kappa_1>0,\ \kappa_2>0
\label{eq:bounded-dense}
\end{equation}
where $\kappa_1$ sets the asymptotic effective-parameter gain from looping: $\kappa_1N_{\text{loop}}$, while $\kappa_2$ controls how quickly this limit is approached as $R$ increases.

\textbf{\textit{Boundary properties.}} The bounded mapping in Eq.~\ref{eq:bounded-dense} recovers the non-looped baseline at $R{=}1$, and approaches a finite effective-parameter bound as $R{\to}\infty$:
\begin{equation}
N_{\text{eff}}^{\text{bounded}}(1)=N,
\qquad
\lim_{R\to\infty}N_{\text{eff}}^{\text{bounded}}(R)
=N+\kappa_1N_{\text{loop}}.
\end{equation}
Thus, looping contributes at most $\kappa_1N_{\text{loop}}$ additional effective parameters,
an asymptotic property of the mapping rather than a guarantee beyond the observed recurrence range.

\paragraph{Loop scaling law.} For a looped transformer with $R$ recurrence loops, we replace the model-size term $N$ in standard scaling law Eq.~\ref{eq:chinchilla} with a recurrence-dependent effective parameter count $N_{\text{eff}}(R)$:
\begin{equation}
\mathcal{L}(N,D,R)
=
A\,N_{\text{eff}}(R)^{\alpha}
+B\,D^{\beta}+c,
\label{eq:dense-loop}
\end{equation}
where $N_{\text{eff}}(R)$ introduces recurrence-specific coefficients alongside the base scaling-law coefficients $\{A,\alpha,B,\beta,c\}$: in Eq.~\ref{eq:recurrence-mappings-revised}, the linear mapping uses linear constant and the power-law mapping uses $\varphi$, while our bounded mapping in Eq.~\ref{eq:bounded-dense} uses $\{\kappa_1,\kappa_2\}$. Figure~\ref{fig:dense-loop}(a) illustrates the normalized effective-parameter gain $(N_{\text{eff}}(R)-N)/N_{\text{loop}}$ under these mappings. The effective parameter count $N_{\text{eff}}(R)$ is distinct from the unrolled parameter count $N_{\text{unroll}}(R)$: the latter determines compute, whereas the former models the effective parameter capacity attributed to weight-tied recurrence.

\textbf{\textit{Reduced form $\mathcal{L}|_{R=1}$.}}
For any recurrence mapping satisfying $N_{\text{eff}}(1)=N$, the loop scaling law in Eq.~\ref{eq:dense-loop} reduces to the standard non-looped scaling law in Eq.~\ref{eq:chinchilla} at $R{=}1$ (as derived in \textbf{Proposition}~\ref{prop:dense-loop-equivalence}).

\paragraph{Comparison of recurrence mappings.}
We derive the fitted laws with the same parametric fitting and experimental sweep over $(R,N,D)$, as detailed in Appendix~\ref{app:staged-fitting}.
Figure~\ref{fig:dense-loop}(b) compares the predictive loss over recurrence for loop scaling laws under different recurrence mappings, where each law is fitted on the sweep over $N$, $D$, $R\leq8$ and extrapolated to the held-out $R{=}16$. The linear mapping substantially overestimates the effective-parameter gain, while the power-law mapping captures diminishing gains but remains overly optimistic beyond the fitted range. In contrast, the bounded mapping captures the loss plateau. The held-out evaluation over unseen $N, D, R$ in Figure~\ref{fig:dense-loop}(c) further shows the law with bounded recurrence mapping achieves the lowest held-out RMSE, indicating it better captures the scaling trends across model size, data, and recurrence.

\subsection{MoE Loop Scaling Law}
\label{sec:loop_moe_scaling_law}

The loop scaling law in Section~\ref{sec:loop_scaling_law} assumes dense looped models. In looped MoE models, sparse routing may select different experts across recurrent passes, allowing tokens to traverse different parameter paths despite weight sharing (Figure~\ref{fig:loopmoe}(a)). Motivated by this expert-path diversity (see analysis in Appendix~\ref{app:expert-path-diversity}), we introduce MoE sparsity as a condition for recurrence mapping, allowing sparsity to modulate the effective-parameter gain from looping.

\begin{figure}[!t]
\centering
\begin{minipage}[c]{0.60\textwidth}
    \centering
    \includegraphics[width=\linewidth]{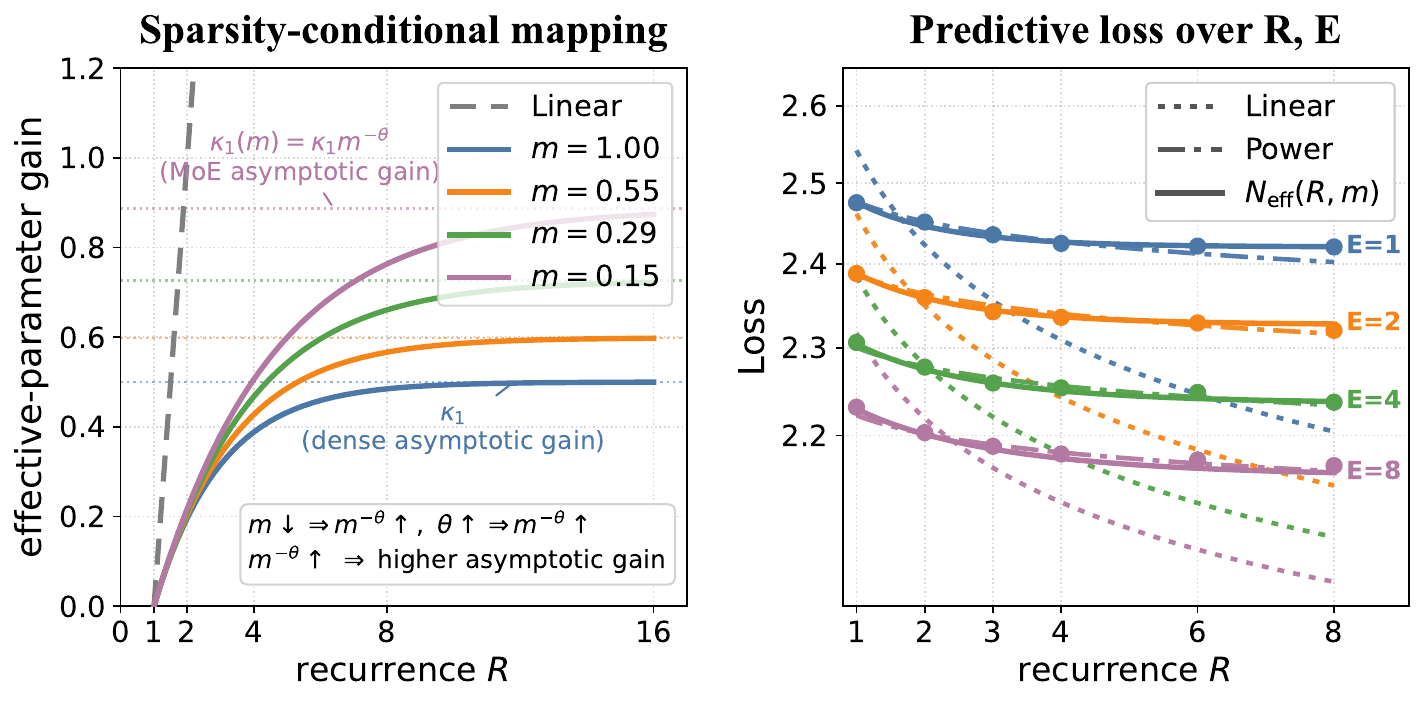}
\end{minipage}\hfill
\begin{minipage}[c]{0.38\textwidth}
    \centering
    {\scriptsize\bfseries MoE recurrence mappings for $N_{\text{eff}}$\par}
    \scriptsize
    \setlength{\tabcolsep}{1.25pt}
    \renewcommand{\arraystretch}{1.0}
    \resizebox{0.96\linewidth}{!}{%
    \begin{tabular}{@{}l|c@{}}
    \toprule
    \textbf{Mapping} & $N_{\text{eff}}-N_{\text{act}}$ \\
    \midrule
    \textbf{Linear} & $(R{-}1)N_{\text{loop}}$ \\
    \textbf{Power law} & $(R^{\varphi}{-}1)N_{\text{loop}}$ \\
    \textbf{Bounded} $N_{\text{eff}}(R)$ & $\kappa_1N_{\text{loop}}\!\left(1{-}e^{-(R-1)/\kappa_2}\right)$ \\
    \textbf{Bounded} $N_{\text{eff}}(R,m)$ & $\kappa_1(m)N_{\text{loop}}\!\left(1{-}e^{-(R-1)/\kappa_2(m)}\right)$ \\
    \bottomrule
    \end{tabular}}\\[5pt]
    {\scriptsize\bfseries Held-out prediction RMSE ($\downarrow$)\par}
    \scriptsize
    \setlength{\tabcolsep}{1.5pt}
    \renewcommand{\arraystretch}{1.0}
    \resizebox{0.82\linewidth}{!}{%
    \begin{tabular}{@{}lcccc@{}}
    \toprule
    \textbf{Evaluation} & \textbf{Linear} & \textbf{Power} & \textbf{$N_{\text{eff}}(R)$} & \textbf{$N_{\text{eff}}(R,m)$} \\
    \midrule
    \textbf{Held-out $R$} & 0.2847 & 0.0480 & 0.0190 & \textbf{0.0100} \\
    \textbf{Held-out $E$} & 0.0896 & 0.0113 & 0.0075 & \textbf{0.0047} \\
    \textbf{Held-out $N$} & 0.0656 & 0.0077 & 0.0065 & \textbf{0.0050} \\
    \textbf{Held-out $D$} & 0.0824 & 0.0074 & 0.0058 & \textbf{0.0043} \\
    \bottomrule
    \end{tabular}}
\end{minipage}
\par
\noindent
\makebox[0.30\textwidth]{\small (a)}%
\makebox[0.30\textwidth]{\small (b)}%
\hfill
\makebox[0.38\textwidth]{\small (c)}
\caption{(a) Illustration of effective-parameter gain with sparsity-conditional recurrence mapping. (b) Predictive loss under different MoE recurrence mappings fitted on runs with $R\leq8$; markers denote observations used for fitting. (c) Formula on MoE recurrence mappings and their effective-parameter gains (top), and their evaluation on held-out recurrence $R$, expert count $E$, model size $N$, and data $D$ with RMSE scores (bottom).}
\label{fig:moe-recurrence-mapping}
\end{figure}

\paragraph{A sparsity-conditional recurrence mapping.}
Let $m=N_{\text{act}}/N_{\text{total}}\in(0,1]$ denote the MoE active-parameter ratio, where smaller $m$ indicates greater sparsity and $m{=}1$ is the dense case. For top-$k$ routing over $E_{\text{route}}$ experts, $m\approx k/E_{\text{route}}$. We extend the dense bounded mapping in Eq.~\ref{eq:bounded-dense} to looped MoE models by replacing $N$ with $N_{\text{act}}$ and parameterizing its coefficients as functions of $m$:
\begin{equation}
N_{\text{eff}}(R,m)
=N_{\text{act}}+\kappa_1(m)N_{\text{loop}}
\left(1-e^{-(R-1)/\kappa_2(m)}\right),
\quad
\kappa_j(m)=\kappa_jm^{-\theta}\ (j=1,2).
\label{eq:bounded-moe-revised}
\end{equation}
The fitted sparsity-scaling exponent $\theta\geq0$ quantifies how strongly sparsity amplifies the parameter gain from recurrence, with $\theta=0$ recovering a sparsity-independent recurrence mapping. As sparsity increases ($m$ decreases), the factor $m^{-\theta}$ lifts the recurrence curve vertically through $\kappa_1(m)$ and stretches it horizontally through $\kappa_2(m)$, thus raising the asymptotic effective-parameter gain that is approached over more recurrent passes, as illustrated in Figure~\ref{fig:moe-recurrence-mapping}(a).

\textbf{\textit{Boundary properties.}}
The sparsity-conditional mapping in Eq.~\ref{eq:bounded-moe-revised} recovers the non-looped baseline at $R{=}1$, and approaches a finite effective-parameter bound as $R{\to}\infty$. It also recovers the dense bounded mapping in Eq.~\ref{eq:bounded-dense} at $m{=}1$, and approaches a linear-in-$R$ limit as $m{\to}0$:
\begin{equation}
\begin{aligned}
N_{\text{eff}}(1,m)&=N_{\text{act}},
&\lim_{R\to\infty}N_{\text{eff}}(R,m)
&=N_{\text{act}}+\kappa_1m^{-\theta}N_{\text{loop}},\\
N_{\text{eff}}(R,1)&=N_{\text{eff}}^{\text{bounded}}(R),
&\lim_{m\to0}N_{\text{eff}}(R,m)
&=N_{\text{act}}+\frac{\kappa_1}{\kappa_2}(R-1)N_{\text{loop}},\quad \theta>0.
\end{aligned}
\label{eq:bounded-moe-boundaries-revised}
\end{equation}
Thus, at any fixed $m{>}0$, looping contributes at most $\kappa_1m^{-\theta}N_{\text{loop}}$ additional effective parameters; in the extreme-sparsity limit $m\to0$ (expert count $E\to\infty$), the gain grows linearly with $R$.
These asymptotic limits are properties rather than guarantees beyond the observed recurrence range.

\paragraph{MoE loop scaling law.}
For a looped MoE model with $R$ recurrence loops and active-parameter ratio $m$, we replace the active-parameter term $N_{\text{act}}$ in MoE scaling law~\citep{clark2022unified,ludziejewski2025joint} with the sparsity-conditional effective parameter count $N_{\text{eff}}(R,m)$ in Eq.~\ref{eq:bounded-moe-revised}:
\begin{equation}
\mathcal{L}(N_{\text{act}},D,R,E,m)
=A\,\hat{E}^{\,\delta}\,N_{\text{eff}}(R,m)^{\,\alpha+\gamma\ln\hat{E}}
+B\,\hat{E}^{\,\omega}\,D^{\,\beta+\zeta\ln\hat{E}}+c,
\label{eq:moe-full}
\end{equation}
where $N_{\text{eff}}(R,m)$ introduces the fitted coefficients $\{\kappa_1,\kappa_2,\theta\}$ alongside the base MoE scaling-law coefficients $\{A,\alpha,B,\beta,c,\delta,\gamma,\omega,\zeta,E_{\text{start}},E_{\text{max}}\}$. Following \citet{clark2022unified}, $\hat{E}$ is a monotonic transformation: $\hat{E}^{-1}=\bigl(E-1+\bigl(E_{\text{start}}^{-1}-E_{\text{max}}^{-1}\bigr)^{-1}\bigr)^{-1}+E_{\text{max}}^{-1}$, where $E$ denotes the number of experts under top-1 routing~\citep{clark2022unified,ludziejewski2025joint}. For top-$k$ routing, we define the effective expert expansion as $E=E_{\text{route}}/k$, consistent with the parameterization of \citet{chen2026mobilemoe}.

\textbf{\textit{Reduced forms.}}
The MoE loop scaling law in Eq.~\ref{eq:moe-full} subsumes prior scaling laws as special cases: (1) $\mathcal{L}|_{E=1}$, at $E=1$, it recovers the dense loop scaling law in Eq.~\ref{eq:dense-loop} after reparameterization; (2) $\mathcal{L}|_{R=1}$, at $R=1$, it recovers the standard MoE scaling law; and (3) $\mathcal{L}|_{R=1,E=1}$, at $E=1, R=1$, it recovers the standard dense scaling law in Eq.~\ref{eq:chinchilla}. Full derivations are given in \textbf{Proposition}~\ref{prop:moe-loop-equivalence}.

\paragraph{Comparison of MoE recurrence mappings.}
We derive the fitted MoE loop scaling laws using the same parametric fitting and experimental sweep over $(R,N_{\text{act}},D,E)$, as detailed in Appendix~\ref{app:staged-fitting}.
Figure~\ref{fig:moe-recurrence-mapping}(b) illustrates the predictive loss over recurrence $R$ and MoE sparsity varied by expert count $E$, where each law is fitted on the sweep over varying $N_{\text{act}}$, $D$, $E$, and $R{\leq}8$. The same fitted laws are evaluated on held-out recurrence $R{=}16$ in Figure~\ref{fig:moe-recurrence-mapping}(c).
As observed in the dense comparison, the linear mapping overestimates recurrence gains and the power-law mapping remains optimistic; importantly, neither captures how saturation varies with sparsity. In contrast, our sparsity-conditional bounded mapping tracks the loss plateau across expert counts.
Figure~\ref{fig:moe-recurrence-mapping}(c) further shows the law with sparsity-conditional bounded mapping in Eq.~\ref{eq:bounded-moe-revised} achieves the lowest held-out RMSE across all held-out axes, as compared to the laws with linear, power-law in Eq.~\ref{eq:recurrence-mappings-revised} and sparsity-independent mappings Eq.~\ref{eq:bounded-dense}.
This confirms our formulated MoE loop scaling law in Eq.\ref{eq:moe-full} has a better predictive fit and captures how sparsity modulates the asymptotic gain from recurrence.
\section{Experiments and Findings}
\label{sec:experiments}

\subsection{Scaling Experiments and Analysis}
\label{sec:scaling-experiments}

\paragraph{Scaling experiments and fitted law.}
We jointly sweep over four scaling axes: model parameters $N_\text{act}\in\{0.3,0.6,1.0\}\text{B}$, training tokens $D\in\{100,200,\ldots,500\}\text{B}$, expert count $E\in\{1,2,4,8,16\}$, and recurrence $R\in\{1,2,3,4,6,8\}$.
We use standard middle-cycle looping~\citep{geiping2025scaling} to loop over the middle block, leaving the first two prelude layers and last two coda layers unshared.
As our formulation is agnostic to specific looping strategies, we leave other strategies to future work.
The full training and law fitting details are provided in Appendix~\ref{app:scaling-exp}.

\begin{figure}[!t]
    \centering
    \def\figbodywidth{1.0}\def\figpairwidth{0.498}
    \begin{minipage}{\figbodywidth\textwidth}
    \centering
    \begin{minipage}[t]{\figpairwidth\linewidth}
    \centering
    \includegraphics[width=\linewidth]{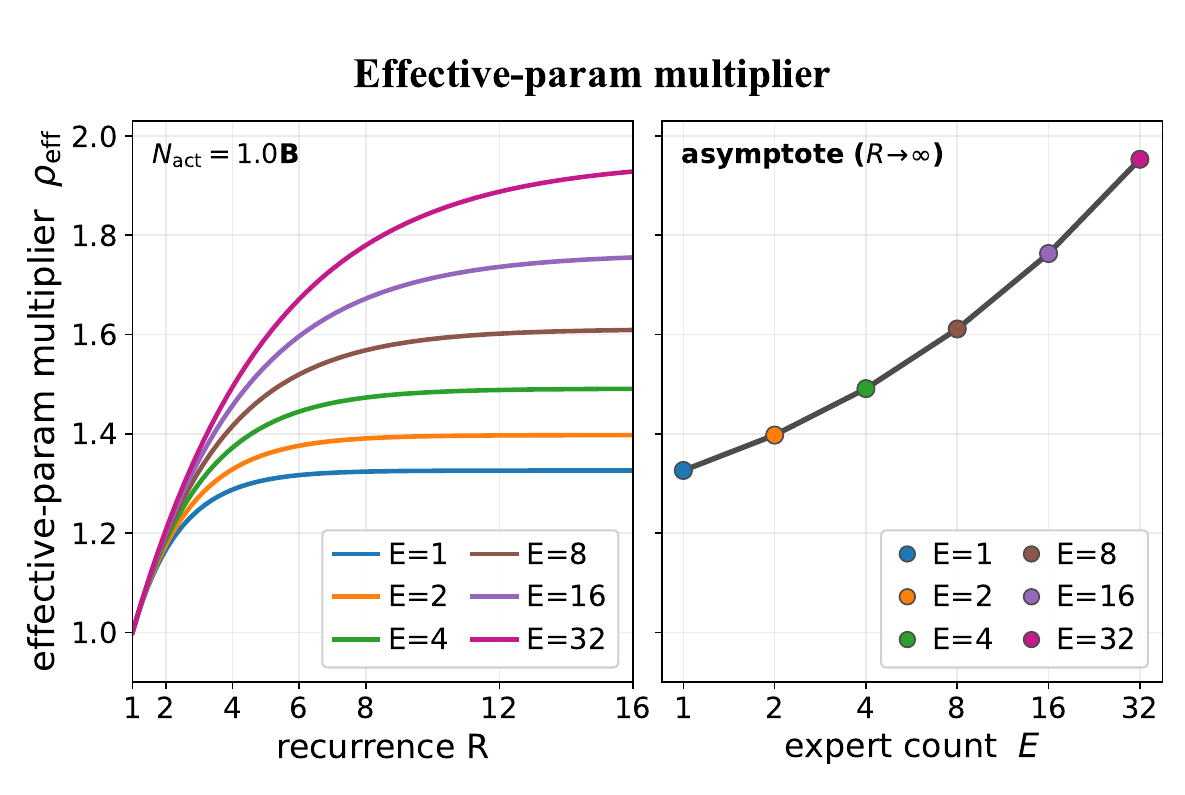}\\
    \makebox[0.5\linewidth]{\small (a)}%
    \makebox[0.5\linewidth]{\small (b)}
    \end{minipage}\hfill
    \begin{minipage}[t]{\figpairwidth\linewidth}
    \centering
    \includegraphics[width=\linewidth]{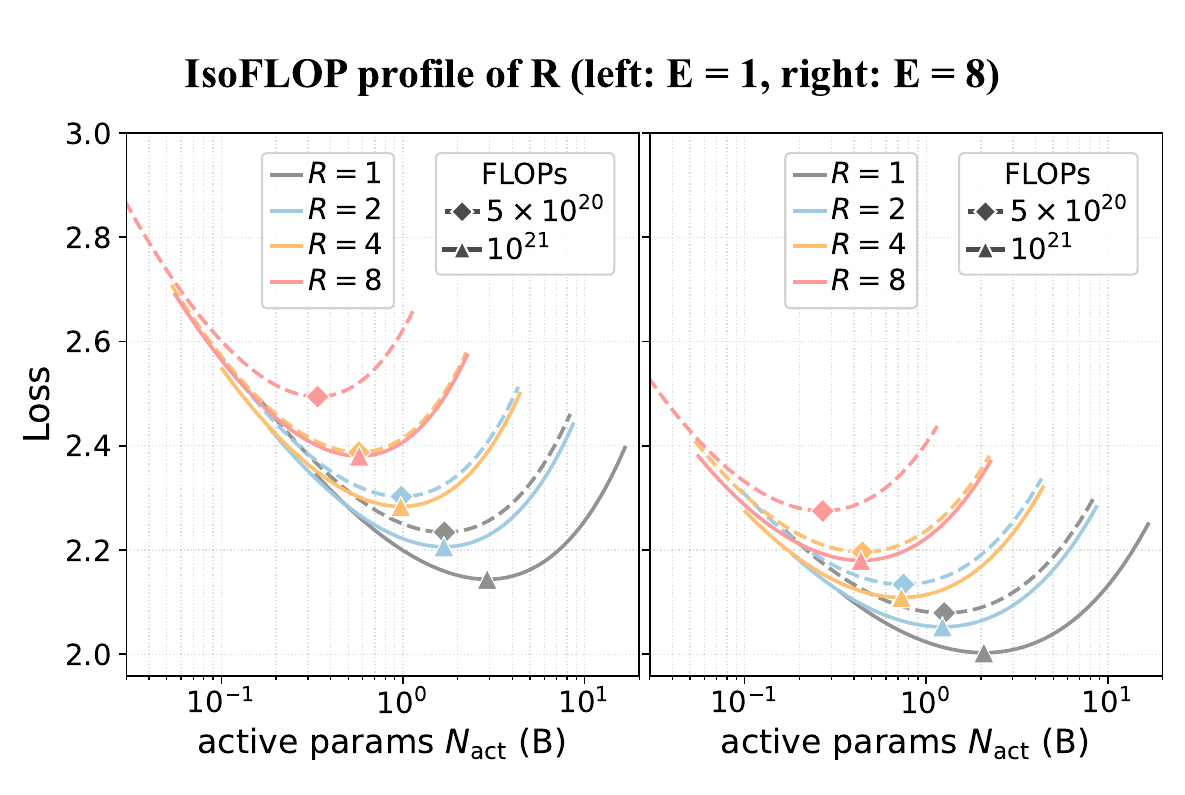}\\
    \makebox[0.5\linewidth]{\small (c)}%
    \makebox[0.5\linewidth]{\small (d)}
    \end{minipage}
    \end{minipage}
\caption{(a,b) Effective-parameter multiplier over recurrence at $N_{\text{act}}{=}1.0$B, and its asymptote over expert count. (c,d) Recurrence IsoFLOP profiles for dense ($E{=}1$) and MoE ($E{=}8$) models; markers denote minima.}
    \label{fig:effective-capacity}
\end{figure}

\paragraph{\textit{How do recurrence and sparsity improve effective capacity from looping?}}
To quantify the effective capacity gain from the fitted MoE loop scaling law, we define the effective-parameter multiplier:
\begin{equation}
\rho_{\text{eff}}(R,m)
=\frac{N_{\text{eff}}(R,m)}{N_{\text{act}}}
=1+\kappa_1(m)\frac{N_{\text{loop}}}{N_{\text{act}}}
\left(1-e^{-(R-1)/\kappa_2(m)}\right).
\label{eq:rho-effective}
\end{equation}
Its asymptote is $\rho_{\text{eff}}(\infty,m)=1+\kappa_1(m)N_{\text{loop}}/N_{\text{act}}$, which characterizes the maximum effective-parameter multiplier relative to $N_{\text{act}}$. Figure~\ref{fig:effective-capacity}(a) shows that, at fixed $N_{\text{act}}$, the effective-parameter multiplier increases with larger recurrence but approaches a finite asymptote, while greater sparsity raises this asymptote and delays saturation over recurrent passes; Figure~\ref{fig:effective-capacity}(b) shows this asymptote itself rises with expert count.
Alternatively, this gain can be measured relative to the recurrent-block parameters $N_{\text{loop}}$ as the effective-parameter gain $g_{\text{eff}}=(N_{\text{eff}}-N_{\text{act}})/N_{\text{loop}}$ (see Figures~\ref{fig:loopmoe}(c),~\ref{fig:dense-loop}(a), and~\ref{fig:moe-recurrence-mapping}(a)), which yields the same scaling trends over recurrence and sparsity, as summarized below.
\begin{findingbox}{\textcolor{navy}{Finding:} \mdseries\color{black} Scaling recurrence adds effective-parameter gain up to a finite, sparsity-dependent asymptote, while scaling sparsity raises and sustains this gain over more recurrent passes.}
\end{findingbox}

\paragraph{\textit{How do recurrence and sparsity push the IsoFLOP frontier?}}
We examine the scaling behavior over recurrence and sparsity on an IsoFLOP basis based on the fitted law.
Figures~\ref{fig:effective-capacity}(c,d) show predicted loss against $N_{\text{act}}$ under fixed training compute: $5{\times}10^{20}$ and $10^{21}$ FLOPs. When increasing sparsity from $E{=}1$ (dense) to $E{=}8$ (MoE), i.e., comparing Figure~\ref{fig:effective-capacity}(c) and (d), lower loss is attained under the same compute budget, indicating that higher sparsity pushes the IsoFLOP performance frontier. When increasing recurrence at fixed sparsity, higher recurrence can outperform lower recurrence at the same active parameter count $N_{\text{act}}$ with additional compute; for example, $R{=}2$ at $10^{21}$ FLOPs achieves lower loss than $R{=}1$ at $5{\times}10^{20}$ FLOPs for both $E{=}1$ and $E{=}8$, as summarized below.

\begin{findingbox}{\textcolor{navy}{Finding:} \mdseries\color{black} Higher sparsity improves the IsoFLOP frontier at matched compute, while recurrence pushes this frontier further at the same parameter size when given more compute.}
\end{findingbox}

\subsection{Compute- and Memory-Optimal Recurrence and Sparsity}
\label{sec:optimal-recurrence-sparsity}

\paragraph{\textit{Given the fitted law, which recurrence and sparsity should be selected under resource constraints?}} Our MoE loop scaling law provides a principled foundation for optimizing architectures under practical resource constraints of training compute, and deployment memory. At fixed weight memory, recurrence $R$ trades additional compute for better performance; whereas at fixed compute, sparsity trades additional weight memory for improved performance.
These complementary tradeoffs motivate our following analyses that apply the fitted law to predict the compute-optimal recurrence at fixed sparsity, memory-optimal sparsity at fixed recurrence, and their joint optimum

\begin{figure}[!t]
    \centering
    \def\figbodywidth{1.0}
    \begin{minipage}{\figbodywidth\textwidth}
    \centering
    \includegraphics[width=\linewidth]{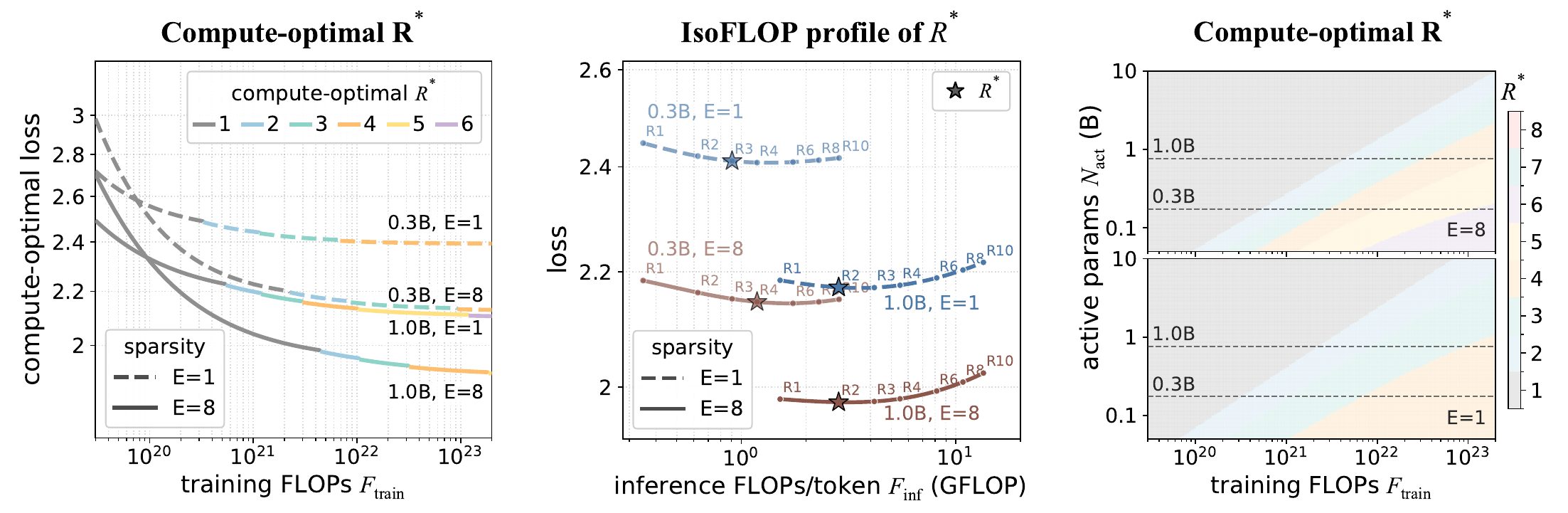}\\
    \makebox[0.333\linewidth]{\small (a)}%
    \makebox[0.333\linewidth]{\small (b)}%
    \makebox[0.333\linewidth]{\small (c)}
    \end{minipage}
   \caption{Predicted compute-optimal recurrence $R^{\star}$ of dense ($E{=}1$) and MoE ($E{=}8$) models. (a) Recurrence with compute-optimal loss across training-compute budgets. (b) IsoFLOP profiles over recurrence at $\bar F_{\text{train}}{=}5{\times}10^{21}$ FLOPs; stars mark $R^{\star}$. (c) $R^{\star}$ across compute and active model size for $E{=}8$ (top) and $E{=}1$ (bottom).}
    \label{fig:compute-optimal-r}
\end{figure}

\paragraph{Compute-optimal recurrence at fixed sparsity.} At fixed model sparsity (fixed expert count $E$), increasing recurrence $R$ leaves the model weights $N_{\text{total}}$ and weight memory $\mathcal{M}_{\text{weight}}$ unchanged, but increases the training compute through unroll parameters $N_{\text{unroll}}(R)$. Therefore, recurrence can be compute-optimal if it achieves the minimal loss among recurrences trained under the same compute, while providing a meaningful loss reduction over the preceding recurrence:
\begin{equation}
R^{\star}=\arg\min_R\mathcal{L}(N_{\text{act}},D,R,E,m),\qquad \text{s.t.}\quad 6N_{\text{unroll}}(R)D=\bar F_{\text{train}},\quad \Delta\mathcal{L}(R)\geq\epsilon.
\label{eq:compute-optimal-r}
\end{equation}
$\bar F_{\text{train}}$ is the given training-compute budget. For $R\geq2$, $\Delta\mathcal{L}(R)=\mathcal{L}_{R-1}{-}\mathcal{L}_R$ measures the predicted loss reduction from the additional recurrent pass $R{-}1{\to}R$. The tolerance $\epsilon$ specifies the minimum meaningful loss reduction. In practice, we set $\epsilon$ to the fitted-law RMSE, below which predicted improvements cannot be reliably distinguished from fitting error. If no pass satisfies $\Delta\mathcal{L}(R)\geq\epsilon$, $R^{\star}{=}1$.
Figure~\ref{fig:compute-optimal-r} shows three trends on compute-optimal $R^{\star}$: (a) $R^{\star}$ increases with more training-compute budget and is higher for smaller active models; (b) along each IsoFLOP profile, the selected $R^{\star}$ lies at or near the loss minimum, with subsequent recurrent passes providing negligible gains; and (c) across the $F_{\text{train}}$--$N_{\text{act}}$ design space, greater sparsity shifts the compute-optimal regime toward higher recurrence. These trends can be summarized below.

\begin{findingbox}{\textcolor{navy}{Finding:} \mdseries\color{black} Larger training compute budgets and greater sparsity favor higher recurrence.}
\end{findingbox}

\begin{figure}[!t]
    \centering
    \def\figbodywidth{1.0}
    \begin{minipage}{\figbodywidth\textwidth}
    \centering
    \includegraphics[width=\linewidth]{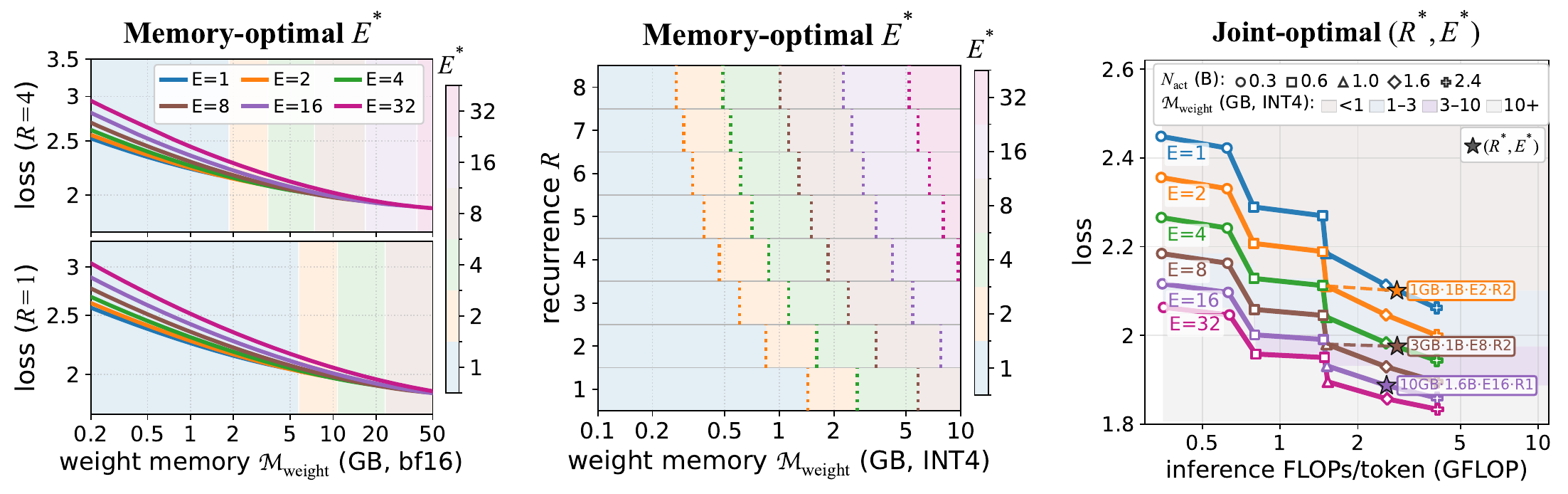}\\
    \makebox[0.333\linewidth]{\small (a)}%
    \makebox[0.333\linewidth]{\small (b)}%
    \makebox[0.333\linewidth]{\small (c)}
    \end{minipage}
\caption{Predicted memory-optimal expert count $E^{\star}$ and joint optimum $(N_{\text{act}}^{\star},E^{\star},R^{\star})$ at $\bar F_{\text{train}}{=}5{\times}10^{21}$ FLOPs. (a) $E^{\star}$ across bf16 weight-memory budgets at $R{=}4$ (top) and $R{=}1$ (bottom). (b) $E^{\star}$ across INT4 weight memory and recurrence. (c) Predicted joint optimum $(N_{\text{act}}^{\star},E^{\star},R^{\star})$ across INT4 weight memory.}
    \label{fig:memory-optimal-e}
\end{figure}

\paragraph{Memory-optimal sparsity at fixed recurrence.}
At fixed recurrence $R$, increasing sparsity through expert count $E$ expands the total parameters $N_{\text{total}}$ and thus weight memory, while leaving the active compute unchanged. Therefore, sparsity can be memory-optimal by selecting the expert count $E$ that achieves the minimal loss when trained under the same compute:
\begin{equation}
E^{\star}=\arg\min_E\mathcal{L}(N_{\text{act}},D,R,E,m),\quad
\text{s.t.}\; F_{\text{train}}=\bar F_{\text{train}},\;
\mathcal{M}_{\text{weight}}\leq M_{\text{budget}},\;
\Delta\mathcal{L}(E)\geq\epsilon.
\label{eq:memory-optimal-e}
\end{equation}
$M_{\text{budget}}$ is the weight-memory budget and $\mathcal{M}_{\text{weight}}=b_wN_{\text{total}}/8$, where $b_w$ is the weight precision in bits. Total deployment memory may include KV cache and other runtime state, but their optimization is outside our scope; we therefore consider weight memory only.
For expert counts $\{E_i\}$ and $i{\geq}2$, $\Delta\mathcal{L}(E_i)=\mathcal{L}_{E_{i-1}}{-}\mathcal{L}_{E_i}$ measures the loss reduction from expert expansion $E_{i-1}{\to}E_i$. The same tolerance $\epsilon$ as in Eq.~\ref{eq:compute-optimal-r} requires each expansion yields a meaningful loss reduction. If no expansion satisfies $\Delta\mathcal{L}(E_i)\geq\epsilon$, $E^{\star}{=}E_{1}{=}1$.
Figure~\ref{fig:memory-optimal-e}(a) and (b) show two trends on memory-optimal $E^{\star}$: (a) $E^{\star}$ increases with more memory budget, while higher recurrence favors larger $E^{\star}$ under the same memory; and (b) across the $\mathcal{M}_{\text{weight}}$--$R$ design space, larger memory budgets and higher recurrence shift the optimal regime toward greater sparsity. These trends are summarized below.

\begin{findingbox}{\textcolor{navy}{Finding:} \mdseries\color{black} Larger weight memory budgets and higher recurrence favor greater sparsity.}
\end{findingbox}

\paragraph{Joint optimum on recurrence and sparsity.} Given a training-compute budget and a memory budget, the recurrence $R$ and sparsity (varied by expert count $E$) can be jointly compute- and memory-optimal if they achieve the minimal loss under the same resource constraints:
\begin{align}
&(N_{\text{act}}^{\star},E^{\star},R^{\star})
=\underset{N_{\text{act}},E,R}{\arg\min}\;
\mathcal{L}(N_{\text{act}},D,R,E,m)
\label{eq:joint-optimum}\\[-0.3em]
&\text{s.t.}\; \text{compute: }6N_{\text{unroll}}(R)D=\bar F_{\text{train}},\;
\Delta\mathcal{L}(R)\geq\epsilon;\;
\text{memory: }\mathcal{M}_{\text{weight}}\leq M_{\text{budget}},\;
\Delta\mathcal{L}(E)\geq\epsilon.
\notag
\end{align}
Given a training compute budget $\bar F_{\text{train}}$ and weight memory budget $M_{\text{budget}}$,
we use the fitted scaling law to identify the predicted optimal looped MoE configuration from a model ladder spanning active model size, expert count, and recurrence (Appendix~\ref{app:scaling-ladder}), which guides model design under resource constraints.
For each candidate configuration $(N_{\text{act}},E,R)$, we set $D=\bar F_{\text{train}}/[6N_{\text{unroll}}(R)]$, retain candidates satisfying the compute and memory constraints, and select the optimum with lowest predicted loss.
Figure~\ref{fig:memory-optimal-e}(c) shows that under fixed training compute, tighter memory favors higher recurrence and smaller expert counts, e.g., selecting a 1.0B active model with $E{=}8$ and $R{=}2$ at 3\,GB; whereas larger memory favors greater sparsity and less recurrence, e.g., selecting a 1.6B active model with $E{=}16$ and $R{=}1$ at 10\,GB, as summarized below.

\begin{findingbox}{\textcolor{navy}{Finding:} \mdseries\color{black} At fixed compute, tight memory favors recurrence, and more memory favors sparsity.}
\end{findingbox}

\subsection{Downstream Scaling with Recurrence and Sparsity}
\label{sec:downstream-scaling}

\paragraph{Empirical scaling on downstream tasks.}
We extend our scaling analysis of looped MoE models to downstream tasks, covering 14 benchmarks spanning five categories: reasoning (BBH, GSM8K), science (ARC-C/E, OpenBookQA), commonsense (HellaSwag, PIQA, SIQA, WinoGrande), reading (BoolQ, DROP), and knowledge (MMLU, Natural Questions, TriviaQA). We report \emph{Overall} as the mean over 14 benchmarks.
Motivated by prior work showing that recurrent depth can preferentially improve latent reasoning~\citep{saunshi2025reasoning,geiping2025scaling,zhu2025scaling}, we also report \empty{Reasoning} as the mean of GSM8K and BBH.
In Figure~\ref{fig:downstream-scaling}, we evaluate downstream performance as a function of training compute, comparing models with active sizes of 0.3B, 0.6B, and 1.0B under varying expert count and recurrence at matched training FLOPs.
The complete benchmark suite and evaluation protocols are provided in Appendix~\ref{app:evaluation}.

\begin{figure}[!t]
    \centering
    \def\figbodywidth{1.0}\def\figpairwidth{0.499}
    \begin{minipage}{\figbodywidth\textwidth}
    \centering
    \begin{minipage}[t]{\figpairwidth\linewidth}
    \centering
    \includegraphics[width=\linewidth]{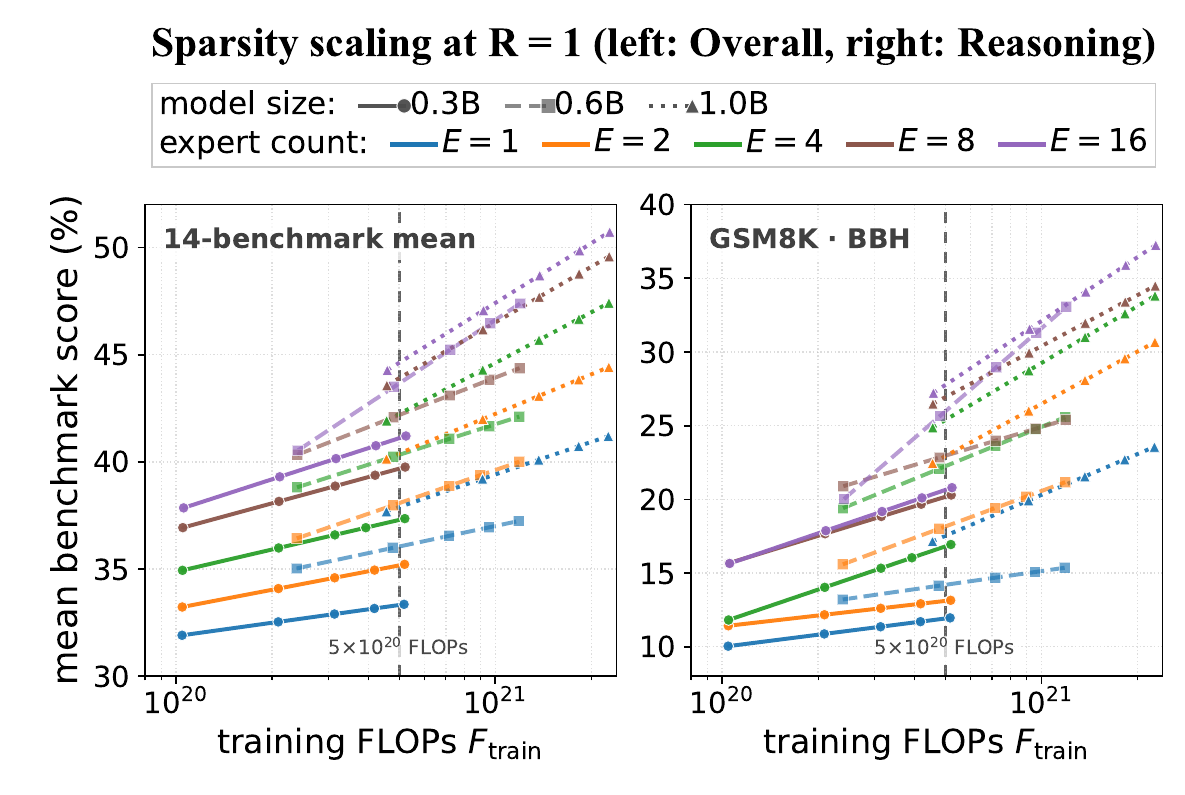}\\
    \makebox[0.5\linewidth]{\small (a)}%
    \makebox[0.5\linewidth]{\small (b)}
    \end{minipage}\hfill
    \begin{minipage}[t]{\figpairwidth\linewidth}
    \centering
    \includegraphics[width=\linewidth]{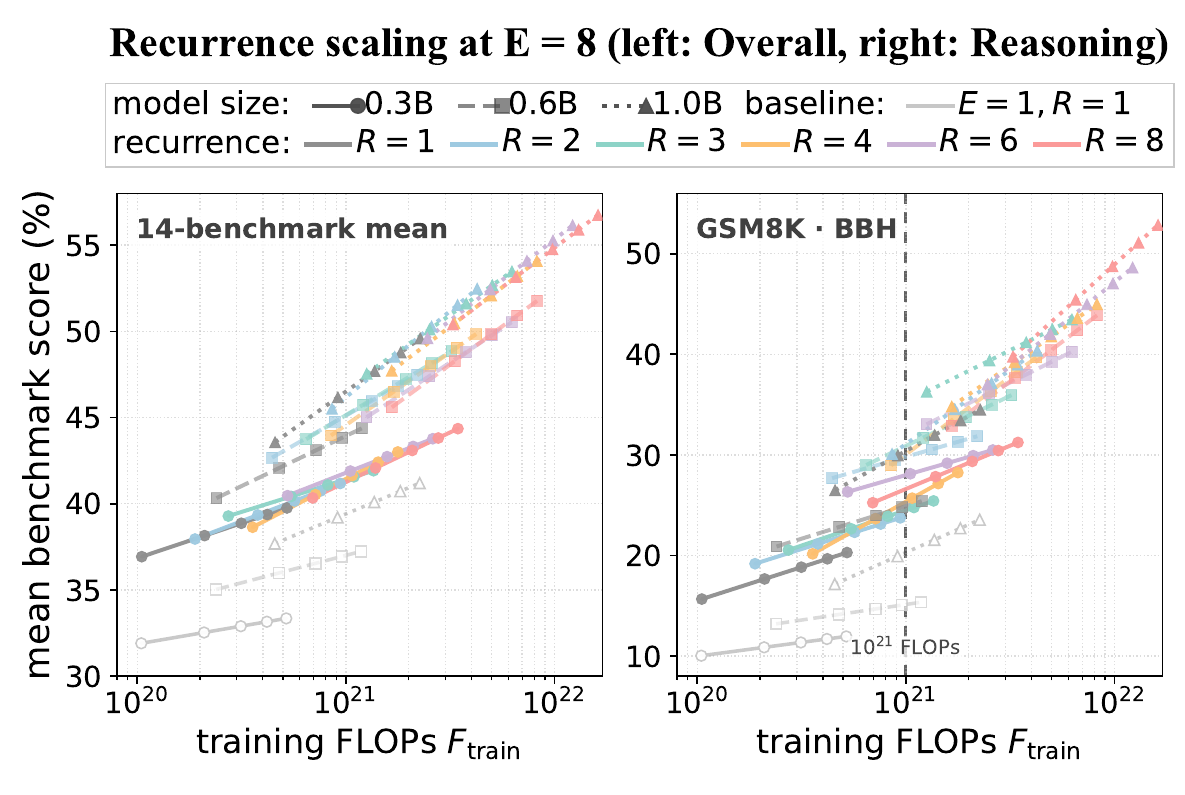}\\
    \makebox[0.5\linewidth]{\small (c)}%
    \makebox[0.5\linewidth]{\small (d)}
    \end{minipage}
    \end{minipage}
    \caption{Downstream scaling of sparsity and recurrence. (a,b) Scaling $E$ at $R{=}1$; (c,d) Scaling $R$ at $E{=}8$, with the dense baseline $E{=}1,R{=}1$. (a,c) report \emph{Overall} performance, and (b,d) report \emph{Reasoning} performance; dashed lines mark compute-matched comparisons.}
    \label{fig:downstream-scaling}
\end{figure}

\paragraph{\textit{How do recurrence and sparsity translate into downstream performance?}} Figure~\ref{fig:downstream-scaling} shows the two scaling axes provide distinct but complementary gains. In panels (a,b), increasing sparsity at fixed $R{=}1$ consistently improves both Overall and Reasoning across all three model sizes. These sparsity gains can offset a ${\sim}3\times$ increase in active model size: at $5{\times}10^{20}$ FLOPs, the 0.3B MoE models with $E{\geq}8$ outperform the larger 1.0B dense model on both Overall and Reasoning. In panels (c,d), increasing recurrence at fixed $E{=}8$ produces more reasoning-oriented gains. These recurrence gains can offset a ${\sim}2\times$ increase in total parameters on Reasoning: at $10^{21}$ FLOPs, the 0.3B model with $R{\geq}4$ matches or exceeds the 0.6B model with $R{=}1$, while the 0.6B model with $R{\geq}3$ matches the 1.0B model with $R{=}1$. Moreover, models that jointly increase sparsity and recurrence ($E{=}8,R{>}1$) consistently outperform the dense baseline ($E{=}1,R{=}1$) at matched training compute, showing that their gains accumulate when scaled together.
Thus, sparsity and recurrence provide complementary routes to parameter-efficient scaling, with sparsity achieving ${\sim}3\times$ active-parameter efficiency on overall performance and recurrence achieving ${\sim}2\times$ total-parameter efficiency on reasoning.

\begin{table}[!t]
\centering
\caption{Downstream results of A0.6B-2.9B MoE and A0.3B-1.3B LoopMoE across test-time recurrence $R$, trained at matched compute FLOPs ${\approx}1.5{\times}10^{22}$. Relative inference compute is normalized to the non-looped baseline. \textcolor{green!60!black}{$\Delta$} reports gains over LoopMoE at $R{=}1$. Evaluation protocols are provided in Appendix~\ref{app:evaluation}.}
\label{tab:operating-recurrence}
\setlength{\tabcolsep}{2pt}
\resizebox{\textwidth}{!}{%
\begin{tabular}{@{}l|c|c|cc ccc cccc cc ccc l@{}}
\toprule
\multirow{2}{*}{\textbf{Model}} &
\multirow{2}{*}{$\mathbf{R}$} &
\multirow{2}{*}{$\displaystyle\frac{F_{\text{inf}}^{\text{0.3B}}(R)}{F_{\text{inf}}^{\text{0.6B}}(1)}$} &
\multicolumn{2}{c}{\textbf{Reasoning}} &
\multicolumn{3}{c}{\textbf{Science}} &
\multicolumn{4}{c}{\textbf{Commonsense Reasoning}} &
\multicolumn{2}{c}{\textbf{Reading}} &
\multicolumn{3}{c}{\textbf{Knowledge}} & \\
\cmidrule(lr){4-5}
\cmidrule(lr){6-8}
\cmidrule(lr){9-12}
\cmidrule(lr){13-14}
\cmidrule(lr){15-17}
& & &
BBH$^3$ & GSM8K$^8$ & ARC-C$^{25}$ & ARC-E & OBQA & HS & PIQA & SIQA & Wino & BoolQ & DROP$^3$ & MMLU$^5$ & NQ$^5$ & TQA$^5$ & \textbf{Overall (\textcolor{green!60!black}{$\Delta$})} \\
\midrule
A0.6B-2.9B MoE & 1 & $1.0\times$ &
29.8 & 42.9 & 50.0 & 66.9 & 40.8 & 68.1 & 77.2 & 51.6 & 64.4 & 70.1 & 39.7 & 46.9 & 13.8 & 39.8 & 50.2 \\
\midrule
\multirow{5}{*}{A0.3B-1.3B LoopMoE}
& 1 & $0.4\times$ &
18.6 & 10.5 & 36.1 & 57.3 & 33.8 & 52.7 & 70.7 & 43.8 & 54.6 & 58.1 & 23.0 & 29.2 & 5.7 & 18.2 & 36.6 \\
& 2 & $0.8\times$ &
26.7 & 30.0 & 42.2 & 60.3 & 38.2 & 59.9 & 72.5 & 46.4 & 58.5 & 62.8 & 33.7 & 37.6 & 8.3 & 24.1 & 43.0\,{\scriptsize\textcolor{green!60!black}{$\boldsymbol{(+6.4)}$}} \\
& 3 & $1.1\times$ &
30.1 & 41.0 & 44.9 & 60.9 & 39.6 & 62.4 & 73.2 & 50.2 & 61.3 & 65.9 & 38.4 & 41.6 & 8.9 & 25.6 & 46.0\,{\scriptsize\textcolor{green!60!black}{$\boldsymbol{(+9.4)}$}} \\
& 4 & $1.5\times$ &
31.5 & 41.5 & 46.5 & 63.9 & 39.0 & 62.9 & 73.9 & 51.1 & 60.6 & 64.4 & 39.3 & 42.2 & 9.8 & 27.1 & 46.7\,{\scriptsize\textcolor{green!60!black}{$\boldsymbol{(+10.1)}$}} \\
& 5 & $1.8\times$ &
31.8 & 41.2 & 47.5 & 64.3 & 40.4 & 63.1 & 74.0 & 51.1 & 61.0 & 64.3 & 39.0 & 42.2 & 9.9 & 27.0 & 46.9\,{\scriptsize\textcolor{green!60!black}{$\boldsymbol{(+10.3)}$}} \\
\bottomrule
\end{tabular}
}%
\end{table}

\paragraph{A practical case study of looped MoE.}
As a practical extension of our scaling analysis on downstream tasks, we explore looped MoE in comparison to an approximately $2\times$ larger non-looped MoE. Specifically, taking an A0.6B-2.9B MoE (0.6B active/2.9B total, $E{=}8,R{=}1$) as the non-looped reference, we use the fitted law and apply our recurrence-selection criterion with the compute constraint (Eq.~\ref{eq:compute-optimal-r}) to derive the recurrence for an A0.3B-1.3B looped MoE (0.3B active/1.3B total, $E{=}8$). We then train the two MoE models at matched compute at trillion-token scale (${\sim}3$T tokens for A0.3B-1.3B LoopMoE and ${\sim}6$T tokens for A0.6B-2.9B MoE), and compare their downstream performance. Further derivation and training details are in Appendix~\ref{app:recurrence-selection}.

\paragraph{\textit{How does a looped MoE compare to a $2\times$ larger non-looped MoE at matched compute?}}
Table~\ref{tab:operating-recurrence} compares A0.3B-1.3B LoopMoE and A0.6B-2.9B MoE at matched training compute. Despite using ${\sim}2\times$ fewer active and total parameters, A0.3B-1.3B LoopMoE at $R\in\{4,5\}$ matches A0.6B-2.9B MoE on reasoning tasks (BBH, GSM8K) by using more inference compute, while trailing on \emph{Overall} performance.
Beyond parameter efficiency, LoopMoE supports test-time scaling on demand: varying recurrence from $R{=}1$ to $R{=}5$ raises inference compute from low to high and boosts \emph{Overall} performance by a substantial $+10.3$ points ($36.6\to46.9$; Table~\ref{tab:operating-recurrence}).

\section{Conclusion}
\label{sec:conclusion}

We introduced \textit{Loop Scaling Laws}, a unified predictive law that jointly models recurrence and sparsity alongside model size and data, closing a pressing gap between MoE and looped model scaling theories. Its bounded, sparsity-conditional mapping captures the diminishing recurrence gains observed within the evaluated range and models how sparsity raises the fitted capacity. The fitted law guides looped MoE model design by selecting recurrence and sparsity under compute and memory constraints, providing practical recipes for resource-constrained settings.

More broadly, our work reframes the scaling design space, positioning recurrence and sparsity as complementary axes for the next-generation parameter-efficient scaling paradigm. As effective scaling relies not only on how many model parameters are added, but also on how effectively additional computation can use them, these two axes should be co-designed according to resource constraints. By combining both axes properly, looped MoE models can advance the performance--efficiency frontier while enabling on-demand test-time scaling through recurrence.

To unlock the greater potential of jointly scaling recurrence and sparsity, we highlight several future directions. For modeling, more advanced looping strategies may be explored to enrich recurrent latent representations and boost the effective-capacity gain further. For scaling, looped MoE models could be generalized across broader scales, spanning from small models on the edge to large models in the cloud. For inference, the memory and latency of looped models could be further improved by optimizing recurrent states or adding early-exit gating, thus scaling recurrence only when necessary.

\bibliography{iclr2027_conference}
\bibliographystyle{iclr2027_conference}

\newpage
\appendix
\section*{APPENDIX}
\renewcommand{\thetable}{\Alph{section}.\arabic{table}}
\renewcommand{\thefigure}{\Alph{section}.\arabic{figure}}
\counterwithin{table}{section}
\counterwithin{figure}{section}
\setcounter{table}{0}
\setcounter{figure}{0}
\section{Reduced Forms of Loop Scaling laws}
\label{app:formulation-properties}

\subsection{Dense Loop Scaling Law}
\label{app:dense-reduced-form}

\begin{proposition}[Dense loop scaling law equivalence]
\label{prop:dense-loop-equivalence}
At $R{=}1$, the dense loop scaling law recovers the standard non-looped scaling law.
\end{proposition}
\begin{proof}\mbox{}\\
\textbf{\textit{Reduced form $\mathcal{L}|_{R=1}$.}} With recurrence fixed at $R{=}1$, the recurrence mapping $N_{\text{eff}}(R)$ (Eq.~\ref{eq:bounded-dense}) recovers the non-looped baseline: $N_{\text{eff}}(1)=N$. Specifically,
\[
\left.N_{\text{eff}}(R)\right|_{R=1}
=N+\kappa_1N_{\text{loop}}
\left(1-e^{-(1-1)/\kappa_2}\right)
=N.
\]
The loop scaling laws reduce to
\begin{equation}
\begin{aligned}
\left.\mathcal{L}(N,D,R)\right|_{R=1}
&=A\,N_{\text{eff}}(1)^\alpha+B\,D^\beta+c\\
&=A\,N^\alpha+B\,D^\beta+c\\
&=\mathcal{L}(N,D).
\end{aligned}
\label{eq:dense-loop-reduced-proposition}
\end{equation}
This reduced form is equivalent to the standard non-looped scaling law in Eq.~\ref{eq:chinchilla}~\citep{hoffmann2022chinchilla,kaplan2020scaling}.
\end{proof}

\subsection{MoE Loop Scaling Law}
\label{app:moe-reduced-forms}

\begin{proposition}[MoE loop scaling law equivalence]
\label{prop:moe-loop-equivalence}
The MoE loop scaling law recovers the dense loop scaling law at $E{=}1$, the non-looped MoE scaling law at $R{=}1$, and the standard dense scaling law at $E{=}1$ and $R{=}1$, with coefficient reparameterization where required.
\end{proposition}
\begin{proof}\mbox{}\\
\textbf{\textit{Reduced form $\mathcal{L}|_{E=1}$.}} With expert count at $E{=}1$, we recover the dense case, for which $m{=}1$ and $N_{\text{act}}{=}N$. Since $\left.\kappa_j(m)\right|_{m=1}=\left.\kappa_jm^{-\theta}\right|_{m=1}=\kappa_j1^{-\theta}=\kappa_j$ for $j=1,2$, the recurrence mapping in Eq.~\ref{eq:bounded-moe-revised} therefore reduces to the dense looped model formula:
\[
\begin{aligned}
\left.N_{\text{eff}}(R,m)\right|_{E=1}
&=\left.\left[N_{\text{act}}+\kappa_1(m)N_{\text{loop}}
\left(1-e^{-(R-1)/\kappa_2(m)}\right)\right]\right|_{m=1,N_{\text{act}}=N}\\
&=N+\kappa_1N_{\text{loop}}
\left(1-e^{-(R-1)/\kappa_2}\right)
=N_{\text{eff}}^{\text{bounded}}(R).
\end{aligned}
\]
Let $\hat{E}_1\equiv\hat{E}|_{E=1}$. We reparameterize the coefficients as $\widetilde{A}=A\hat{E}_1^{\delta}$, $\widetilde{\alpha}=\alpha+\gamma\ln\hat{E}_1$, $\widetilde{B}=B\hat{E}_1^{\omega}$, and $\widetilde{\beta}=\beta+\zeta\ln\hat{E}_1$. The MoE loop scaling law in Eq.~\ref{eq:moe-full} then reduces to
\begin{equation}
\begin{aligned}
\left.\mathcal{L}(N_{\text{act}},D,R,E,m)\right|_{E=1}
&=A\,\hat{E}_1^{\delta}N_{\text{eff}}(R,1)^{\alpha+\gamma\ln\hat{E}_1}
+B\,\hat{E}_1^{\omega}D^{\beta+\zeta\ln\hat{E}_1}+c\\
&=\widetilde{A}\,N_{\text{eff}}^{\text{bounded}}(R)^{\widetilde{\alpha}}
+\widetilde{B}\,D^{\widetilde{\beta}}+c\\
&=\mathcal{L}(N,D,R).
\end{aligned}
\label{eq:moe-reduced-e1-proposition}
\end{equation}
This reduced form is equivalent to the loop scaling law in Eq.~\ref{eq:dense-loop} for dense looped transformers.

\par\noindent\textbf{\textit{Reduced form $\mathcal{L}|_{R=1}$.}} With recurrence fixed at $R{=}1$, the recurrence mapping in Eq.~\ref{eq:bounded-moe-revised} reduces to
\[
\begin{aligned}
\left.N_{\text{eff}}(R,m)\right|_{R=1}
&=\left.\left[N_{\text{act}}+\kappa_1(m)N_{\text{loop}}
\left(1-e^{-(R-1)/\kappa_2(m)}\right)\right]\right|_{R=1}\\
&=N_{\text{act}}+\kappa_1(m)N_{\text{loop}}
\left(1-e^0\right)=N_{\text{act}}.
\end{aligned}
\]
The MoE loop scaling law in Eq.~\ref{eq:moe-full} then reduces to
\begin{equation}
\begin{aligned}
\left.\mathcal{L}(N_{\text{act}},D,R,E,m)\right|_{R=1}
&=A\,\hat{E}^{\,\delta}N_{\text{eff}}(1,m)^{\,\alpha+\gamma\ln\hat{E}}
+B\,\hat{E}^{\,\omega}D^{\,\beta+\zeta\ln\hat{E}}+c\\
&=A\,\hat{E}^{\,\delta}N_{\text{act}}^{\,\alpha+\gamma\ln\hat{E}}
+B\,\hat{E}^{\,\omega}D^{\,\beta+\zeta\ln\hat{E}}+c\\
&=\mathcal{L}(N_{\text{act}},D,E).
\end{aligned}
\label{eq:moe-reduced-r1-proposition}
\end{equation}
This reduced form is equivalent to the MoE scaling law for non-looped MoE models~\citep{ludziejewski2025joint}.

\par\noindent\textbf{\textit{Reduced form $\mathcal{L}|_{R=1,E=1}$.}} Setting both $R{=}1$ and $E{=}1$ gives $m{=}1$, $N_{\text{act}}{=}N$, and $\left.\kappa_j(m)\right|_{m=1}=\left.\kappa_jm^{-\theta}\right|_{m=1}=\kappa_j1^{-\theta}=\kappa_j$ for $j=1,2$. The recurrence mapping in Eq.~\ref{eq:bounded-moe-revised} thus reduces to the dense non-looped baseline:
\[
\begin{aligned}
\left.N_{\text{eff}}(R,m)\right|_{R=1,E=1}
&=\left.\left[N_{\text{act}}+\kappa_1(m)N_{\text{loop}}
\left(1-e^{-(R-1)/\kappa_2(m)}\right)\right]\right|_{R=1,m=1,N_{\text{act}}=N}\\
&=N+\kappa_1N_{\text{loop}}\left(1-e^0\right)=N.
\end{aligned}
\]
Using the same reparameterized coefficients $\widetilde{A}=A\hat{E}_1^{\delta}$, $\widetilde{\alpha}=\alpha+\gamma\ln\hat{E}_1$, $\widetilde{B}=B\hat{E}_1^{\omega}$, and $\widetilde{\beta}=\beta+\zeta\ln\hat{E}_1$ as in Eq.~\ref{eq:moe-reduced-e1-proposition}, the MoE loop scaling law in Eq.~\ref{eq:moe-full} then reduces to
\begin{equation}
\begin{aligned}
\left.\mathcal{L}(N_{\text{act}},D,R,E,m)\right|_{R=1,E=1}
&=A\,\hat{E}_1^{\delta}N_{\text{eff}}(1,1)^{\alpha+\gamma\ln\hat{E}_1}
+B\,\hat{E}_1^{\omega}D^{\beta+\zeta\ln\hat{E}_1}+c\\
&=\widetilde{A}\,N^{\widetilde{\alpha}}
+\widetilde{B}\,D^{\widetilde{\beta}}+c\\
&=\mathcal{L}(N,D).
\end{aligned}
\label{eq:moe-reduced-dense-proposition}
\end{equation}
This reduced form is equivalent to the standard non-looped scaling law in Eq.~\ref{eq:chinchilla}~\citep{hoffmann2022chinchilla,kaplan2020scaling}.
\end{proof}

\section{Expert-Path Diversity Across Recurrence}
\label{app:expert-path-diversity}

\paragraph{Expert-path diversity.}
The sparsity-conditional mapping is motivated by MoE sparse routing: a token may access different experts across recurrent passes. We measure this directly from the router decisions of trained looped MoE models. For an input token $t$, looped MoE layer $\ell$, and pass $r$, let $\mathcal{E}_{t,\ell,r}$ be the top-$k$ experts selected from $E_{\text{route}}$ routed experts. We define the expert-path diversity as
\begin{equation}
\Psi(R)
=\frac{1}{k}\,
\mathbb{E}_{t,\ell}
\left[
\left|\bigcup_{r=1}^{R}\mathcal{E}_{t,\ell,r}\right|
\right],
\label{eq:expert-path-diversity}
\end{equation}
where the union deduplicates experts reused across recurrent passes, the expectation averages over tokens and the $n_{\ell,\text{loop}}$ looped layers, and normalizing by $k$ expresses the expert-path diversity per active expert. At $R{=}1$ (non-looped) or $E_{\text{route}}{=}k{=}1$ (dense), $\Psi{=}1$ by definition, as a single pass or a single expert does not have path diversity over recurrent passes.

\begin{table}[!h]
\centering
\caption{Expert-path diversity $\Psi$ across varying recurrence $R$ and expert count $E$ (mean $\pm$ std).}
\label{tab:expert-path-diversity}
\scriptsize
\setlength{\tabcolsep}{5pt}
\begin{tabular}{c|cccc}
\toprule
$E$ & $R{=}1$ & $R{=}2$ & $R{=}4$ & $R{=}8$ \\
\midrule
1 & $1.0000\pm0.0000$ & $1.0000\pm0.0000$ & $1.0000\pm0.0000$ & $1.0000\pm0.0000$ \\
2 & $1.0000\pm0.0000$ & $1.1212\pm0.0692$ & $1.2165\pm0.0465$ & $1.5088\pm0.0529$ \\
4 & $1.0000\pm0.0000$ & $1.1917\pm0.0974$ & $1.2973\pm0.0703$ & $1.5300\pm0.0785$ \\
8 & $1.0000\pm0.0000$ & $1.2099\pm0.0868$ & $1.4159\pm0.0901$ & $1.7035\pm0.0861$ \\
\bottomrule
\end{tabular}
\end{table}

\paragraph{Empirical measures of $\Psi$.}
We report $\Psi$ for looped models across varying recurrence and expert count, where $\Psi$ is computed on the same validation set of 800 prompts from different tasks, e.g., math, code, and knowledge.
Table~\ref{tab:expert-path-diversity} shows that at fixed recurrence, expert-path diversity $\Psi$ increases with larger expert count consistently. This indicates that greater sparsity lets each recurrent pass reach a broader set of expert parameters. As $\Psi$ is only an empirical routing statistic, we introduce sparsity as a condition for modeling the modulated effective-parameter gain (Eq.~\ref{eq:bounded-moe-revised}).

\section{Scaling Experiments and Parametric Fitting}
\label{app:scaling-exp}

\subsection{Training Details}
\label{app:pretraining-setup}

\paragraph{Model training setups.}
We follow the experimental setup similar to recent work on MoE scaling laws~\citep{chen2026mobilemoe}, using the same pre-training data. All models are trained from scratch on the same open-licensed, web-heavy data mixture, supplemented with math, code, knowledge, and science data. We train with a context length of 2,048 and sequence packing, tied input-output embeddings, and RoPE with base frequency $500{,}000$. We use the llama tokenizer with a vocabulary of 202k. We use AdamW with $\beta_1{=}0.9$, $\beta_2{=}0.95$, $\epsilon{=}10^{-15}$, weight decay $0.1$, and gradient clipping at $1.0$. Model weights are trained in BF16, while optimizer states, gradients, and MoE router computations are maintained in FP32. All scaling-sweep experiments run on 8 nodes (64 NVIDIA H100 96\,GB GPUs) with a global batch size of 3,072 and a sequence length of 2,048, corresponding to approximately $6.29$ million tokens per batch.

\paragraph{Looped MoE setups.}
For MoE models, we adopt sigmoid gating with per-token top-$k$ normalization, auxiliary-loss-free load balancing with bias-update rate $\lambda_{\text{lb}}{=}10^{-3}$, router z-loss with coefficient $\lambda_z{=}10^{-4}$, and drop-and-pad token dispatch with capacity factor $1.5$.
For the scaling sweep experiments used for loop scaling laws in Sections~\ref{sec:loop_scaling_law}--\ref{sec:scaling-experiments}, we apply the next-token prediction loss to the output of the final recurrent pass. For the looped MoE in Table~\ref{tab:operating-recurrence}, we add per-loop supervision~\citep{bae2024relaxed} to allow the same trained checkpoint to be evaluated at different recurrence at test time.

\subsection{Model Ladder and Scaling Sweep}
\label{app:scaling-ladder}

\paragraph{Model ladder setups.}
We construct a model ladder with five active model scales $N_{\text{act}}\in\{0.3{,}0.6{,}1.0{,}1.6{,}2.4\}\text{B}$. Each rung fixes its own model configuration $d_m$, $n_h$, $n_{\ell}$, $n_{\ell,\text{loop}}$, and $N_{\text{loop}}$. Here, $d_m$ is the model dimension; $n_h$ and $n_{\text{KV}}$ are the numbers of attention and KV heads; $n_{\ell}$ and $n_{\ell,\text{loop}}$ are the numbers of total and recurrent layers; $d_h$ is the head dimension; $N_{\text{act}}, N_{\text{total}}$ denote the active and total parameter counts and $N_{\text{loop}}$ is the parameter count of the recurrent block. All configurations use $n_{\text{KV}}{=}4$ and $d_h{=}64$, with $n_h=d_m/d_h$. Since we adopt the standard middle-block looping~\citep{geiping2025scaling}, the first two and last two layers remain unshared, giving $n_{\ell,\text{loop}}=n_{\ell}-4$.
Table~\ref{tab:scaling-ladder} summarizes the model ladder setups.

\begin{table}[!h]
\centering
\caption{Model scaling ladder over five model scales. Parameter counts are in billions (B). $N_{\text{act}}$ and $N_{\text{total}}$ include embedding parameters, whereas $N_{\text{loop}}$ includes only parameters in the recurrent block.}
\label{tab:scaling-ladder}
\fontsize{7.5}{9}\selectfont
\setlength{\tabcolsep}{5pt}
\begin{tabular}{c|cccc|c|ccccc}
\toprule
& \multicolumn{4}{c|}{Architecture dimensions} & & \multicolumn{5}{c}{$N_{\text{total}}(E)$ (B)} \\
\cmidrule(lr){2-5}
\cmidrule(lr){7-11}
$N_{\text{act}}$ (B) & $d_m$ & $n_h$ & $n_{\ell}$ & $n_{\ell,\text{loop}}$ & $N_{\text{loop}}$ (B) & $E{=}1$ & $E{=}2$ & $E{=}4$ & $E{=}8$ & $E{=}16$ \\
\midrule
0.3 & 768  & 12 & 20 & 16 & 0.1 & 0.3 & 0.5 & 0.8 & 1.3 & 2.5 \\
0.6 & 1024 & 16 & 26 & 22 & 0.3 & 0.6 & 0.9 & 1.6 & 2.9 & 5.5 \\
1.0 & 1280 & 20 & 32 & 28 & 0.7 & 1.0 & 1.6 & 2.9 & 5.4 & 10.5 \\
1.6 & 1536 & 24 & 38 & 34 & 1.2 & 1.6 & 2.7 & 4.8 & 9.1 & 17.7 \\
2.4 & 1792 & 28 & 44 & 40 & 1.8 & 2.4 & 4.1 & 7.5 & 14.3 & 27.8 \\
\bottomrule
\end{tabular}
\end{table}

\paragraph{Scaling sweep runs.}
Our experimental scaling sweep uses $N_{\text{act}}\in\{0.3,0.6,1.0\}\text{B}$ across training tokens $D\in\{100,200,300,400,500\}\text{B}$, expert counts $E\in\{1,2,4,8,16\}$ which vary sparsity and active-parameter ratios $m$, and recurrences $R\in\{1,2,3,4,6,8\}$. Following \citet{kaplan2020scaling}, the fitted parameter counts exclude embeddings.

\paragraph{Model optimization design space.}
For the resource-constrained joint optimization of recurrence and sparsity in Figure~\ref{fig:memory-optimal-e}(c) (Section~\ref{sec:optimal-recurrence-sparsity}), we evaluate the fitted law over five model scales on the model ladder: $N_{\text{act}}\in\{0.3{,}0.6{,}1.0{,}1.6{,}2.4\}\text{B}$, expert counts $E\in\{1,2,4,8,16,32\}$, and recurrences $R\in\{1,\ldots,10\}$. This gives 300 candidate configurations, where $N_{\text{act}}\in\{1.6,2.4\}\text{B}$, expert counts $E{>}16$, and recurrences $R{>}8$ lie in the extrapolation range of the fitted law.

\subsection{Parametric Fitting}
\label{app:staged-fitting}

\paragraph{Fitting the dense loop scaling laws.}
For fitting the dense loop scaling law, we instantiate the effective parameter count $N_{\text{eff}}(R)$ in Eq.~\ref{eq:dense-loop} using one of the following recurrence mappings:
\begin{equation}
N_{\text{eff}}(R)=
\begin{cases}
N+(R-1)N_{\text{loop}}, & \text{linear},\\
N+(R^\varphi-1)N_{\text{loop}}, & \text{power law},\\
N+\kappa_1N_{\text{loop}}
\left(1-e^{-(R-1)/\kappa_2}\right), & \text{bounded}.
\end{cases}
\label{eq:dense-mappings-appendix}
\end{equation}
Here, $N$ is the stored model parameter count and $N_{\text{loop}}$ is the parameter count of the recurrent block, as defined in Section~\ref{sec:prelim}. The linear mapping coincides with the unrolled parameter count $N_{\text{unroll}}(R)$, whereas the power-law and bounded mappings model the effective-parameter gain from looping. The linear mapping introduces no recurrence-specific fitted coefficient, while the power-law and bounded mappings introduce $\varphi$ and $\{\kappa_1,\kappa_2\}$, respectively.

At $R{=}1$, the dense loop scaling law recovers the standard non-looped scaling law (Proposition~\ref{prop:dense-loop-equivalence}, Eq.~\ref{eq:dense-loop-reduced-proposition}). We therefore derive the loop scaling law using a staged fitting procedure: first establish the non-looped scaling law with base coefficients $\{A,\alpha,B,\beta,c\}$ using the $R{=}1$ subset of the scaling sweep across model size $N$ and data $D$, and then extend it to the loop scaling law to jointly estimate all coefficients using the scaling sweep across $R$, $N$, and $D$. For a fair comparison in Figure~\ref{fig:dense-loop}, we apply the same fitting procedure to the loop scaling law under the linear and power-law mappings (Eq.~\ref{eq:recurrence-mappings-revised}) and the bounded mapping (Eq.~\ref{eq:bounded-dense}), and evaluate each formulation on the same held-out $R$, $N$, and $D$, where the held-out slices are $R{=}16$, $N{=}1.0\text{B}$, and $D{\in}(400,500]\text{B}$.
Figure~\ref{fig:dense-loop}(c) reports the RMSE on each held-out slice after fitting on the remaining runs.

\paragraph{Fitting the MoE loop scaling laws.}
For fitting the MoE loop scaling law, we instantiate the effective parameter count $N_{\text{eff}}(R,m)$ in Eq.~\ref{eq:moe-full} using one of the following recurrence mappings:
\begin{equation}
N_{\text{eff}}(R,m)=
\begin{cases}
N_{\text{act}}+(R-1)N_{\text{loop}}, & \text{linear},\\
N_{\text{act}}+(R^\varphi-1)N_{\text{loop}}, & \text{power law},\\
N_{\text{act}}+\kappa_1N_{\text{loop}}
\left(1-e^{-(R-1)/\kappa_2}\right), & \text{bounded},\\
N_{\text{act}}+\kappa_1(m)N_{\text{loop}}
\left(1-e^{-(R-1)/\kappa_2(m)}\right), & \text{sparsity-conditional}.
\end{cases}
\label{eq:moe-mappings-appendix}
\end{equation}
For the sparsity-conditional mapping, $\kappa_j(m)=\kappa_jm^{-\theta}$ for $j\in\{1,2\}$. Here, $N_{\text{act}}$ is the active parameter count, $N_{\text{loop}}$ is the active parameter count of the recurrent block, and $m=N_{\text{act}}/N_{\text{total}}$ is computed from the active and total parameter counts as defined in Section~\ref{sec:prelim}. The linear mapping coincides with $N_{\text{unroll}}(R)$, while the linear, power-law, and bounded recurrence mappings are independent of $m$; their outer MoE scaling laws retain the dependence on expert count $E$. The linear mapping introduces no recurrence-specific fitted coefficient, while the power-law, bounded, and sparsity-conditional mappings introduce $\varphi$, $\{\kappa_1,\kappa_2\}$, and $\{\kappa_1,\kappa_2,\theta\}$, respectively.

At $R{=}1$, the MoE loop scaling law recovers the standard non-looped MoE scaling law (Proposition~\ref{prop:moe-loop-equivalence}, Eq.~\ref{eq:moe-reduced-r1-proposition}). We derive the MoE loop scaling law using a staged fitting procedure: first establish the non-looped MoE scaling law with base coefficients $\{A,\alpha,B,\beta,c,\delta,\gamma,\omega,\zeta,E_{\text{start}},E_{\text{max}}\}$ using the $R{=}1$ subset of the scaling sweep across active model size $N_{\text{act}}$, training tokens $D$, and expert count $E$; and then we extend it to the MoE loop scaling law to jointly estimate all coefficients using the full sweep across the four scaling axes: $R$, $N_{\text{act}}$, $D$, and $E$. For a fair comparison in Figure~\ref{fig:moe-recurrence-mapping}, we apply the same fitting procedure to the MoE loop scaling law under the linear, sparsity-independent power-law, bounded, and sparsity-conditional recurrence mappings, and evaluate each formulation on the same held-out $R$, $E$, $N_{\text{act}}$, and $D$, where the held-out slices are $R{=}16$, $E{=}16$, $N_{\text{act}}{=}1.0\text{B}$, and $D{\in}(400,500]\text{B}$. In Figure~\ref{fig:moe-recurrence-mapping}(c), we report the RMSE scores on the held-out runs of recurrence $R$, expert count $E$, model $N_{\text{act}}$, and data $D$, after fitting each law on the remaining runs.

\paragraph{Numerical optimization and final fitted coefficients.}
We fit the final MoE loop scaling law using more than 1,000 loss observations from the scaling sweep across four axes $N_{\text{act}},D,E,R$ (Appendix~\ref{app:scaling-ladder}).
We initialize the scaling-law fit using \texttt{scipy.optimize.curve\_fit} with nonlinear least squares and an MSE objective, and then refine the fit using \texttt{scipy.optimize.minimize} with L-BFGS-B optimization and a log-Huber objective. Table~\ref{tab:moe-fit-appendix} reports the final fitted coefficients of the MoE Loop Scaling Law in Eq.~\ref{eq:moe-full}, using all the scaling sweep runs over recurrence $R$, expert count $E$, model $N_{\text{act}}$, and data $D$.

\begin{table}[!h]
\centering
\caption{Fitted coefficients of MoE Loop Scaling Law in Eq.~\ref{eq:moe-full}. RMSE: root-mean-square error of the fit.}
\label{tab:moe-fit-appendix}
\scriptsize
\setlength{\tabcolsep}{3pt}
\resizebox{\linewidth}{!}{%
\begin{tabular}{@{}rrrrr|rrrrrr|rrr|r@{}}
\toprule
\multicolumn{5}{c|}{\textbf{Base scaling law coefficients}} &
\multicolumn{6}{c|}{\textbf{MoE coefficients}} &
\multicolumn{3}{c|}{\textbf{Recurrence coefficients}} &
\multicolumn{1}{c}{\textbf{Fit}} \\
\cmidrule(lr){1-5}\cmidrule(lr){6-11}\cmidrule(lr){12-14}\cmidrule(l){15-15}
$A$ & $\alpha$ & $B$ & $\beta$ & $c$ &
$\delta$ & $\gamma$ & $\omega$ & $\zeta$ & $E_{\text{start}}$ & $E_{\max}$ &
$\kappa_1$ & $\kappa_2$ & $\theta$ & \textbf{RMSE} \\
\midrule
0.8085 & $-0.1951$ & 3.4291 & $-0.7548$ & 1.3555 &
$-0.1935$ & $-0.0137$ & $-0.2056$ & 0.0881 & 1.3174 & 57.1201 &
0.3682 & 1.4037 & 0.3286 &
0.0037 \\
\bottomrule
\end{tabular}}
\end{table}

\paragraph{Sparsity exponent ablation.}
We ablate a less constrained variant of Eq.~\ref{eq:bounded-moe-revised} with separate sparsity exponents, $\kappa_1(m)=\kappa_1m^{-\theta_1}$ and $\kappa_2(m)=\kappa_2m^{-\theta_2}$. The RMSE remains $0.0037$ on the final fitted law.
As adding extra exponent yields no meaningful gain, we adopt the shared-$\theta$ formulation.

\subsection{Bootstrap Results of the MoE Loop Scaling Law}
\label{app:bootstrap-uncertainty}

To quantify uncertainty in our fitted MoE loop scaling law in Table \ref{tab:moe-fit-appendix}, we use a bootstrapping protocol similar to prior works~\citep{hoffmann2022chinchilla,ludziejewski2025joint}. We repeatedly sample 90\% of the observations without replacement 100 times, and refit the law on each sampled subset. For the fitted coefficients, we report the full-data estimates and 90\% bootstrap percentile confidence intervals (CIs), defined by the 5th and 95th percentiles over the 100 refit runs.

\begin{table}[!h]
\centering
\caption{Bootstrap results on the MoE loop scaling law. (a) Coefficients fitted on full-data or the 90\% sampled subsets with intervals spanning the 5th--95th percentiles over 100 refit runs. (b) Derived joint optimum under the compute and memory budgets (with INT4 weights) used in Figure~\ref{fig:memory-optimal-e}(c) on the 100 refit runs.}
\label{tab:moe-bootstrap-results}
\scriptsize
\begin{minipage}[t]{0.36\linewidth}
\centering
\textbf{(a) Coefficient uncertainty}\par\smallskip
\setlength{\tabcolsep}{4pt}
\begin{tabular}{@{}lcc@{}}
\toprule
\textbf{Coefficient} & \textbf{Estimate} & \textbf{90\% CI} \\
\midrule
$\kappa_1$ & 0.3682 & $[0.3658,\,0.3723]$ \\
$\kappa_2$ & 1.4037 & $[1.3896,\,1.4243]$ \\
$\theta$   & 0.3286 & $[0.3183,\,0.3350]$ \\
\bottomrule
\end{tabular}
\end{minipage}\hfill
\begin{minipage}[t]{0.61\linewidth}
\centering
\textbf{(b) Joint-optimum uncertainty}\par\smallskip
\setlength{\tabcolsep}{4pt}
\begin{tabular}{@{}ccccc@{}}
\toprule
\textbf{Compute budget} & \textbf{Memory budget} & \textbf{$N_{\text{act}}^{\star}$ 90\% CI} & \textbf{$E^{\star}$ 90\% CI} & \textbf{$R^{\star}$ 90\% CI} \\
\midrule
$5{\times}10^{21}$ & 1 GB  & $[1.0,1.0]\text{B}$ & $[2,2]$   & $[2,2]$ \\
$5{\times}10^{21}$ & 3 GB  & $[1.0,1.0]\text{B}$ & $[8,8]$   & $[2,2]$ \\
$5{\times}10^{21}$ & 10 GB & $[1.6,1.6]\text{B}$ & $[16,16]$ & $[1,1]$ \\
\bottomrule
\end{tabular}
\end{minipage}
\end{table}

Table~\ref{tab:moe-bootstrap-results}(a) presents the bootstrap results for coefficients of the sparsity-conditional recurrence mapping. The narrow intervals of 100 refits indicate these coefficient estimates remain stable under bootstrapping.
Table~\ref{tab:moe-bootstrap-results}(b) shows the 90\% bootstrap CIs of the joint optimum for resource-constrained architecture selection in Figure~\ref{fig:memory-optimal-e}(c) (Section~\ref{sec:optimal-recurrence-sparsity}). The collapsed intervals indicate the selected configurations remain the same across 100 refits, and confirm that the fitted law yields robust estimates for architecture optimization under compute and memory budgets.

\section{Downstream Scaling and Evaluation}
\label{app:downstream-evaluation}

\subsection{Downstream Evaluation Protocols}
\label{app:evaluation}

\begin{table}[!h]
\centering
\caption{Evaluation setups of 14 foundational benchmarks. Task names and metrics follow \textsc{lm-eval}.}
\label{tab:eval-config}
\scriptsize
\setlength{\tabcolsep}{2.5pt}
\begin{tabular}{@{}>{\raggedright\arraybackslash}p{0.22\textwidth}>{\raggedright\arraybackslash}p{0.13\textwidth}>{\raggedright\arraybackslash}p{0.188\textwidth}c>{\raggedright\arraybackslash}p{0.292\textwidth}@{}}
\toprule
\textbf{Category} & \textbf{Benchmark} & \textbf{Task name} & \textbf{$n$-shot} & \textbf{Metric} \\
\midrule
\multirow{2}{=}{Reasoning}
 & BBH & \texttt{leaderboard\_bbh} & 3 & \texttt{acc\_norm} \\
 & GSM8K & \texttt{gsm8k\_cot} & 8 & \mbox{\texttt{exact\_match,flexible-extract}} \\
\midrule
\multirow{3}{=}{Science}
 & ARC-C & \texttt{arc\_challenge} & 25 & \texttt{acc\_norm} \\
 & ARC-E & \texttt{arc\_easy} & 0 & \texttt{acc\_norm} \\
 & OBQA & \texttt{openbookqa} & 0 & \texttt{acc\_norm} \\
\midrule
\multirow{4}{=}{\mbox{Commonsense Reasoning}}
 & HellaSwag & \texttt{hellaswag} & 0 & \texttt{acc\_norm} \\
 & PIQA & \texttt{piqa} & 0 & \texttt{acc\_norm} \\
 & SIQA & \texttt{social\_iqa} & 0 & \texttt{acc} \\
 & WinoGrande & \texttt{winogrande} & 0 & \texttt{acc} \\
\midrule
\multirow{2}{=}{Reading}
 & BoolQ & \texttt{boolq} & 0 & \texttt{acc} \\
 & DROP & \texttt{drop} & 3 & \texttt{f1} \\
\midrule
\multirow{3}{=}{Knowledge}
 & MMLU & \texttt{mmlu} & 5 & \texttt{acc} \\
 & NQ & \texttt{nq\_open} & 5 & \texttt{exact\_match} \\
 & TQA & \texttt{triviaqa} & 5 & \texttt{exact\_match} \\
\bottomrule
\end{tabular}
\end{table}

\paragraph{Evaluation protocols.}
We evaluate the models in Section~\ref{sec:downstream-scaling} on 14 widely used benchmarks spanning five core competencies: (1) reasoning (BBH~\citep{bbh}, GSM8K~\citep{gsm8k}), (2) science (ARC-C/E~\citep{arc}, OpenBookQA~\citep{openbookqa}), (3) commonsense (HellaSwag~\citep{hellaswag}, PIQA~\citep{piqa}, SIQA~\citep{siqa}, WinoGrande~\citep{winogrande}), (4) reading (BoolQ~\citep{boolq}, DROP~\citep{drop}), and (5) knowledge (MMLU~\citep{mmlu}, Natural Questions~\citep{naturalquestions}, TriviaQA~\citep{triviaqa}). We use the Language Model Evaluation Harness (\href{https://github.com/EleutherAI/lm-evaluation-harness}{\textsc{lm-eval}})~\citep{eval-harness} with the \href{https://github.com/vllm-project/vllm}{vLLM} backend, bfloat16 model precision, and greedy decoding for generation-based tasks. Detailed benchmark-specific few-shot settings and metrics are given in Table~\ref{tab:eval-config}.
For compactness in Table~\ref{tab:operating-recurrence}, we use the following benchmark abbreviations: BBH: BIG-Bench Hard, ARC-C/E: ARC-Challenge/Easy, OBQA: OpenBookQA, HS: HellaSwag, Wino: WinoGrande, NQ: Natural Questions, TQA: TriviaQA.

\subsection{Compute-Matched Trillion-Token Experiment}
\label{app:recurrence-selection}

\paragraph{Scaling-law-guided model design.}
In Section~\ref{sec:downstream-scaling}, we compare a looped MoE model with a non-looped MoE reference under matched training compute. Given fixed compute FLOPs, we derive the compute-optimal recurrence for the looped MoE model by evaluating each candidate recurrence on an IsoFLOP basis, as summarized in Algorithm~\ref{alg:isoflop-recurrence}.

\begin{algorithmbox}[label={alg:isoflop-recurrence},fontupper=\footnotesize,fonttitle=\small\bfseries,left=4pt,right=4pt,top=2pt,bottom=2pt,before skip=3pt,after skip=3pt,before upper={\raggedright\setlength{\abovedisplayskip}{2pt}\setlength{\belowdisplayskip}{2pt}\setlength{\abovedisplayshortskip}{1pt}\setlength{\belowdisplayshortskip}{1pt}}]{IsoFLOP recurrence selection}
\textbf{Input:} compute budget $\bar F_{\text{train}}$ or reference model $(N_{\text{ref}},D_{\text{ref}})$; target looped MoE architecture $(N_{\text{act}},N_{\text{loop}},E,m)$; recurrence set $\mathcal{R}$; tolerance $\epsilon$.
\begin{enumerate}[leftmargin=*,itemsep=1pt,topsep=1pt,parsep=0pt,partopsep=0pt]
    \item If a reference model is provided, set
    $\bar F_{\text{train}}=6N_{\text{ref}}D_{\text{ref}}$.
    \item For each $R\in\mathcal{R}$ of the target looped MoE, compute its unrolled parameter count and compute-matched training-token budget:
    \[
    N_{\text{unroll}}(R)=N_{\text{act}}+(R-1)N_{\text{loop}},
    \qquad
    D_{\text{target}}^{R}=\bar F_{\text{train}}/[6N_{\text{unroll}}(R)].
    \]
    \item Compute the looping effective parameter count and the predicted loss using Eqs.~\ref{eq:bounded-moe-revised} and~\ref{eq:moe-full}:
    \[
    \begin{aligned}
    N_{\text{eff}}(R,m)
    &=N_{\text{act}}+\kappa_1(m)N_{\text{loop}}
    \left(1-e^{-(R-1)/\kappa_2(m)}\right),\\
    \mathcal{L}_R
    &=\mathcal{L}(N_{\text{act}},D_{\text{target}}^{R},R,E,m).
    \end{aligned}
    \]
    \item Select the compute-optimal recurrence according to
    Eq.~\ref{eq:compute-optimal-r}:
    \[
    \begin{aligned}
    R^{\star}&=\arg\min_R\mathcal{L}(N_{\text{act}},D_{\text{target}}^{R},R,E,m),\\
    \text{s.t.}\quad&6N_{\text{unroll}}(R)D_{\text{target}}^{R}=\bar F_{\text{train}},
    \quad \Delta\mathcal{L}(R)\geq\epsilon.
    \end{aligned}
    \]
\end{enumerate}

\textbf{Output:} Compute-optimal recurrence $R^{\star}$.
\end{algorithmbox}
We set $\epsilon$ to the fitted-law RMSE, which measures the predictive discrepancy in the model loss. A marginal improvement below this resolution is treated as indistinguishable from fitting error and thus insufficient to justify the cost of an additional recurrent pass.

\paragraph{Matched-compute training and test-time evaluation.}
For Table~\ref{tab:operating-recurrence}, the procedure above selects $R^{\star}{=}5$ for the A0.3B-1.3B LoopMoE ($E{=}8$), which we compare with an A0.6B-2.9B non-looped MoE ($E{=}8,R{=}1$). The models are trained for ${\sim}3$T and ${\sim}6$T tokens, respectively, which correspond to approximately $1.5{\times}10^{22}$ FLOPs.
We calculate the compute FLOPs as $F_{\text{train}}=6N_{\text{unroll}}(R)D$, where $D$ is the training tokens and $N_{\text{unroll}}(R)$ is the unrolled parameter count at recurrence $R$. A0.3B-1.3B LoopMoE is trained as described in Appendix~\ref{app:pretraining-setup}.
As LoopMoE gives valid output per recurrent pass, we evaluate the model by varying recurrence $R\in\{1,2,\ldots,5\}$ at test time, to switch inference effort from low to high on demand.
To compare the inference compute of A0.3B-1.3B LoopMoE and A0.6B-2.9B non-looped MoE, we report their relative inference compute ratio as $F_{\text{inf}}^{\text{0.3B}}(R)/F_{\text{inf}}^{\text{0.6B}}(1)$, where $F_{\text{inf}}(R)=2N_{\text{unroll}}(R)$.

\end{document}

%% file: math_commands.tex
\usepackage{amsmath,amsfonts,bm}

\def\eqref#1{equation~\ref{#1}}

\def\1{\bm{1}}

\DeclareMathAlphabet{\mathsfit}{\encodingdefault}{\sfdefault}{m}{sl}
\SetMathAlphabet{\mathsfit}{bold}{\encodingdefault}{\sfdefault}{bx}{n}

